\documentclass{article}

\usepackage{hyperref}
\usepackage{url}
\usepackage{microtype}
\usepackage{graphicx}
\usepackage{booktabs} %
\usepackage{colortbl}
\usepackage{tabulary}
\usepackage{adjustbox}
\usepackage{multirow}
\usepackage{makecell} 
\usepackage{caption}
\usepackage{subcaption}
\usepackage{enumitem}
\usepackage{xcolor}

\usepackage[nointegrals]{wasysym}

\newcommand{\draftonly}[1]{#1}
\renewcommand{\draftonly}[1]{}

\definecolor{ggreen}{HTML}{3cba54}
\definecolor{gblue}{HTML}{4885ed}

\usepackage{tikz}

\newcommand\tf[1]{\textbf{#1}}
\newcommand\ttt[1]{\texttt{#1}}

\newcommand{\mask}{\texttt{[MASK]}}
\newcommand{\firstsent}{\ttt{<}$S_1$\ttt{>}}
\newcommand{\secondsent}{\ttt{<}$S_2$\ttt{>}}
\newcommand{\sent}{\ttt{<}$S_1$\ttt{>}}

\usepackage[accepted]{icml2023}

\usepackage{amsmath}
\usepackage{amssymb}
\usepackage{mathtools}
\usepackage{amsthm}

\usepackage[capitalize,noabbrev]{cleveref}

\theoremstyle{plain}
\newtheorem{theorem}{Theorem}[section]

\newtheorem{lemma}[theorem]{Lemma}

\theoremstyle{definition}
\newtheorem{definition}[theorem]{Definition}
\newtheorem{assumption}[theorem]{Assumption}
\theoremstyle{remark}
\newtheorem{remark}[theorem]{Remark}

\usepackage{amsmath,amsfonts,bm}

\def\eqref#1{equation~\ref{#1}}

\def\1{\bm{1}}

\def\eps{{\epsilon}}

\DeclareMathAlphabet{\mathsfit}{\encodingdefault}{\sfdefault}{m}{sl}
\SetMathAlphabet{\mathsfit}{bold}{\encodingdefault}{\sfdefault}{bx}{n}

\def\gA{{\mathcal{A}}}

\def\gC{{\mathcal{C}}}

\def\gF{{\mathcal{F}}}

\def\gH{{\mathcal{H}}}

\def\gK{{\mathcal{K}}}

\def\gM{{\mathcal{M}}}
\def\gN{{\mathcal{N}}}

\def\gS{{\mathcal{S}}}

\def\gU{{\mathcal{U}}}
\def\gV{{\mathcal{V}}}
\def\gW{{\mathcal{W}}}
\def\gX{{\mathcal{X}}}

\def\sI{{\mathbb{I}}}

\def\sN{{\mathbb{N}}}

\def\sR{{\mathbb{R}}}

\newcommand{\R}{\mathbb{R}}

\DeclareMathOperator*{\argmax}{arg\,max}
\DeclareMathOperator*{\argmin}{arg\,min}

\DeclareMathOperator{\sign}{sign}

\newcommand{\mup}{\mu\text{P}}
\newcommand{\epssign}{\epsilon\text{-}\mathrm{sign}}

\usepackage[textsize=tiny]{todonotes}

\icmltitlerunning{A Kernel-Based View of Language Model Fine-Tuning}

\newcommand{\klora}{\gK_{\text{LoRA}}}

\newcommand{\ksgd}{\gK^{\text{(SGD)}}}
\newcommand{\ksigngd}{\gK^{\text{(SignGD)}}}
\newcommand{\kasigngd}{\gK^{\text{(A-SignGD)}}}
\newcommand{\ksgdlora}{\ksgd_{\text{LoRA}}}

\begin{document}

\twocolumn[
\icmltitle{A Kernel-Based View of Language Model Fine-Tuning}

\begin{icmlauthorlist}
\icmlauthor{Sadhika Malladi}{xxx}
\icmlauthor{Alexander Wettig}{xxx}
\icmlauthor{Dingli Yu}{xxx}
\icmlauthor{Danqi Chen}{xxx}
\icmlauthor{Sanjeev Arora}{xxx}
\end{icmlauthorlist}

\icmlaffiliation{xxx}{Department of Computer Science, Princeton University, Princeton, NJ, USA}

\icmlcorrespondingauthor{Sadhika Malladi}{smalladi@cs.princeton.edu}

\icmlkeywords{Machine Learning, ICML}

\vskip 0.3in
]

\printAffiliationsAndNotice{}  %

\begin{abstract}
It has become standard to solve NLP tasks by  fine-tuning pre-trained language models (LMs), especially in low-data settings. There is minimal theoretical understanding of empirical success, e.g., why fine-tuning a model with $10^8$ or more parameters on a couple dozen training points does not result in overfitting.
 We investigate  whether 
 the Neural Tangent Kernel (NTK)---which originated as a model to study the gradient descent dynamics of infinitely wide networks with suitable random initialization---{\em describes} fine-tuning of pre-trained LMs.  This study was inspired by the decent performance of NTK  for computer vision tasks \citep{wei2022more}.
 We extend the NTK formalism to Adam and use Tensor Programs~\citep{yang2020tensor3} to characterize conditions under which the NTK lens may describe fine-tuning updates to pre-trained language models.
 Extensive experiments on 14 NLP tasks validate our theory and show that formulating the downstream task as a masked word prediction problem through prompting often induces kernel-based dynamics during fine-tuning.
Finally, we use this kernel view to propose an explanation for the success of parameter-efficient subspace-based fine-tuning methods.\footnote{Our code and pre-computed kernels are publicly available at \href{https://github.com/princeton-nlp/LM-Kernel-FT}{https://github.com/princeton-nlp/LM-Kernel-FT}.}
\end{abstract}

\section{Introduction}
It is now customary to solve most supervised natural language processing (NLP)  tasks such as topic classification and textual entailment  by fine-tuning a pre-trained language model (e.g., \cite{devlin2018bert,liu2019roberta,clark2020electra,raffel2020exploring,joshi2019spanbert}). 
We lack theoretical understanding of this fine-tuning paradigm. 
Why do we not see over-fitting when fine-tuning a very large language model  using a couple dozen instances of the supervised task? 
Why is fine-tuning so sensitive to details such as whether or not we include a prompt (e.g., adding ``It was [great/terrible]'' for sentiment analysis \citep{schick2020exploiting,gao2020making}? Why does restricting optimization to a low-rank subspace of model parameters \citep{hu2021lora,li2018measuring,aghajanyan2021intrinsic} still 
result in performance comparable to
full fine-tuning?
Answering such questions requires understanding how the sequence of parameter updates changes in various scenarios, e.g.,  the addition of a prompt, or the introduction of randomly initialized parameters.
The current theory of deep learning, at first sight, seems too primitive to address such questions, especially since fine-tuning has to start from a parameter initialization inherited from pre-training. 

Recently, \citet{wei2022more} suggested replacing fine-tuning with Neural Tangent Kernel (NTK), an idea invented for the study of infinite-width deep neural networks \citep{jacot2018neural,du2018gradientdescent} and previously applied to solving vision tasks with infinitely wide ConvNets \citep{arora2019exact}.
They note that the NTK can be defined for any neural model $f$ and any initialization $\theta_0$ by representing an input $\xi$ by the gradient it induces $\nabla f(\xi;\theta_0)$, which yields a kernel matrix: 
\begin{equation}
    \gK(\xi, \xi') = \langle \nabla f(\xi;\theta_0), \nabla f(\xi';\theta_0)\rangle.
\end{equation} 
This kernel is well-defined for any parameter vector $\theta_0$.
However, for an infinite-width network initialized with $\theta_0$ sampled from a suitably-scaled Gaussians, it can be shown that the kernel matrix is unchanged during gradient descent, which turns the classification task  into a form of kernel regression with respect to this kernel \citep{jacot2018neural}.  
In the fine-tuning setting, however, the initialization $\theta_0$ is inherited from the pre-trained network, and not sampled from the Gaussian distribution.
Nevertheless, \cite{wei2022more} found that kernel regression using this ``empirical NTK'' (eNTK) defined with the inherited $\theta_0$ performs  well, achieving classification accuracy within $6\%$ absolute of actual fine-tuning on several image recognition tasks.
In other words, their work hints that mathematical understanding of the fine-tuning phenomenon (e.g., its sample efficiency) could go via the theory of kernel classifiers.

The current paper furthers an empirical and theoretical understanding of the pre-training and fine-tuning (FT) paradigm for NLP tasks. 
Our contributions are:
 \begin{enumerate}[leftmargin=0.05\textwidth]
    \item \textbf{We formally extend the standard NTK theory developed for gradient descent to characterize kernel-based dynamics when training with Adam.} We propose and rigorously prove the correctness of a new kernel formula relying on the sign of the gradient to describe early-stage training (e.g., fine-tuning) with Adam (\Cref{sec:adam_kernel}).
    \item \textbf{We formally extend infinite-width analysis to account for a pre-trained initialization and characterize conditions under which fine-tuning can exhibit kernel behavior.} Using insights into the importance of prompting, we formally prove the existence of a rigorous mechanism through which prompt-based FT of complex architectures (e.g., Transformers) can exhibit kernel behavior (\Cref{sec:theory}). Analysis proceeds in the context of networks whose widths go to infinity (i.e., through the Tensor Programs framework), but unlike standard NTK theory, it allows a non-random initialization (i.e., one that results from pre-training). %
    \item \textbf{We perform an extensive empirical analysis on 14 diverse NLP tasks to reveal when and to what extent fine-tuning exhibits kernel behavior.} We find that using a meaningful prompt is crucial for the eNTK to achieve good performance, suggesting that prompting induces a well-characterized optimization benefit for fine-tuning. Further experiments reveal that the trajectory of prompt-based FT can often be \emph{described} by kernel-based dynamics when the eNTK succeeds  (\Cref{sec:experiments}). 
    \item \textbf{We straightforwardly apply the kernel view of FT dynamics to formally analyze the success of fine-tuning methods that update in a low-rank subspace of model parameters (e.g., LoRA, ~\cite{hu2021lora}).} These results in \Cref{sec:parameter_efficient} highlight how a kernel-based understanding of FT can aid in the practical design and theoretical analysis of efficient variants.
\end{enumerate}

\section{Related Work} 
\paragraph{Kernel view of training.} 
The infinite-width limit is a well-studied theoretical model for deep network optimization. 
\citet{jacot2018neural} introduced NTK to capture training a deep and infinitely wide neural network from a random initialization.
Subsequent experiments showed that the kernels underperformed for standard tasks \citep{arora2019exact} but performed well on small datasets (i.e., hundreds of examples) \citep{arora2019harnessing}.
Many works \citep{allenzhu2018learning,allenzhu2018convergence,arora2019finegrained,du2018gradient,du2018gradientdescent,li2018learning,zou2018stochastic,cao2019generalization} have since applied this lens to understand the optimization and generalization behavior of deep networks.
However, such analyses 
do not directly apply to the pre-training and fine-tuning framework because (1) the network trained during FT is inherited and non-random; and (2) LMs are often trained with Adam, and the NTK formula only describes training an infinitely wide network with SGD.
In this work, we handle a non-random (i.e., pre-trained) initialization by assuming that the pre-training task is sufficiently related to the downstream task (\Cref{def:solvable_task}), and we derive new kernels to model early-stage training with Adam (\Cref{sec:adam_kernel}). %

\vspace{-0.1in}
\paragraph{Theory of self-supervised learning and transfer learning.}
Several existing theoretical works on transfer learning study the performance of linear probing on a representation to provide guarantees on various tasks related to the original training data \citep{du2020fewshot,tripuraneni2020theory,wu2020understanding}.
\citet{chua2021how} show that regularized fine-tuning in a meta-learning setting exhibits kernel behavior if the pre-training and downstream tasks are closely related.
Along similar lines, \citet{mu2020gradients,maddox2021fast,achille2021lqf} suggest through experiments and theory that gradient-based features, corresponding to a linearization of fine-tuning, can perform well on visual downstream tasks.
We characterize when kernel dynamics describe fine-tuning a pre-trained masked language model on downstream language understanding tasks. 

\citet{saunshi2020mathematical} study autoregressive language models to rigorously characterize why prompting can improve zero-shot task performance, but their analysis precludes an investigation of FT.
We focus on the masked language model pretraining objective, but it is worth noting that there are many works  \citep{arora2019theoretical,tosh2020contrastive,tosh2020contrastiveestimation,lee2020predicting,tsai2020selfsupervised} studying transfer when pre-training with a contrastive objective.
However, experiments on language modeling \citep{abnar2021exploring} and contrastive learning \citep{saunshi2022understanding} recently demonstrated that properties of transfer between self-supervised pre-training and supervised FT cannot be fully captured by model-agnostic analyses that directly relate the pre-training and downstream task errors.
Kernel theory provides a principled optimization- and architecture-aware framework to analyze FT.

\vspace{-0.1in}
\paragraph{Optimization of Transformers.} 
\looseness-1 Several works \citep{zhang2019why,liu2020understanding,li2022robust} have documented issues with optimizing Transformer-based architectures with SGD instead  of Adam. 
To study the unique properties of optimizing transformers with Adam, we derive a new kernel formula (\Cref{thm:asigngd_kernel}) to capture early-stage training with Adam.
\Cref{tab:main_prompted} compares the performance of this kernel to FT with Adam and SGD.

\vspace{-0.1in}
\paragraph{Variants of fine-tuning methods.} A standard way of fine-tuning pre-trained LMs as introduced in \cite{radford2018improving,devlin2018bert} is to add a linear classifier on top of a pre-trained encoder and update all the parameters together. Subsequent work~\citep{schick2020exploiting,gao2020making} formulated downstream tasks as a language modeling problem (i.e., prompt-based FT) and demonstrated empirical success in low-data scenarios (see \citet{liu2021pretrain} for a comprehensive survey). Another line of research studies parameter-efficient fine-tuning methods in which only a subset of model parameters are updated~\citep{lester2021power,zaken2021bitfit,li2021prefix} or the parameters updates are restricted to a low-dimensional subspace~\citep{hu2021lora,aghajanyan2021intrinsic}.
We show that good eNTK performance arises only when studying prompt-based FT in \Cref{sec:experiments} (\Cref{fig:prompt_vs_noprompt}) and we later show in \Cref{sec:parameter_efficient} that subspace-based FT methods such as LoRA \citep{hu2021lora} have a simple interpretation through the kernel.

\section{Preliminaries}\label{sec:prelims}

\subsection{Pre-Training and Fine-Tuning Paradigm}\label{sec:pt_ft_prelims}
We focus our attention on masked language models (MLMs), such as BERT~\citep{devlin2018bert} and RoBERTa \citep{liu2019roberta}, which are trained to minimize the cross-entropy loss on independently predicting masked tokens (i.e., a $|\gV|$-way classification task, where $\gV$ is the vocabulary).
Given a text input $s$ of length $T$ from the pre-training distribution $\gS_{\text{PT}}$, replace a small percentage (e.g., 15\%) of tokens with \verb+[MASK]+ tokens.
This masked input is then fed into the representation function $h:\gS_{\text{PT}}\to T\times \sR^n$ (e.g., a Transformer encoder) to produce a low-dimensional contextual embedding for each position in the input.
The contextual embeddings are independently multiplied by a classifier head (i.e., word embeddings) $\Phi\in\sR^{n\times |\gV|}$ to produce logits that will be used to compute the probability of a token filling each masked position.

Using a pre-trained model to solve downstream tasks effectively has been a highly active area of research. 
We focus on fine-tuning (FT) methods, which adapt the pre-trained model to a new input distribution $\gS_{\text{FT}}$ using additional training on the $C$-way downstream classification task.

\begin{enumerate}[leftmargin=0.05\textwidth]
	\item \textbf{Standard FT} \citep{devlin2018bert,liu2019roberta}: To solve a $C$-way downstream classification task, initialize and learn\footnote{In our experiments, Standard FT corresponds to initializing $\Gamma$ at the linear probing solution (i.e., training $\Gamma$ on the downstream task while freezing all other layers) and then performing FT. We do this because when FT exhibits kernel behavior (\Cref{def:kernel_regime}), it finds solutions close to initialization, and we hypothesize that the $\Gamma$ learned during FT is closer to the linear probing solution than a random initialization.} a new classifier head $\Gamma:\sR^n\to\sR^C$ on top of the contextual  \verb+[CLS]+ embedding, denoted $h_{\mathtt{[CLS]}}$. In this case, the model output $f:\gS_{\text{FT}}\to\sR^C$ for the eNTK construction is $f(s) = \Gamma(h_{\mathtt{[CLS]}}(s))$. 
	\item \textbf{Prompt-based FT} \citep{schick2020exploiting,gao2020making}: Add a natural language prompt (e.g. ``This is \verb+[MASK]+.'') in addition to the downstream task input, and use the pre-trained MLM to fill in the masked token. Compute the logits over task-relevant words (e.g., ``great'' and ``terrible'') using the corresponding columns of $\Phi$, denoted $\tilde{\Phi}\in\sR^{n\times C}$. 
	These logits will serve as surrogates to solve the downstream task.
	In this case, the model output $f:\gS_{\text{FT}}\to\sR^C$ for the eNTK construction is $f(s) = \tilde\Phi^\top h_{\mathtt{[MASK]}}(s)$.
\end{enumerate}

\subsection{Kernel Behavior}
We consider a neural network $f(\xi;\theta)$ that takes input $\xi$ and computes a scalar output\footnote{Note that for $C$-way classification, $f$ outputs a vector in $\R^C$. 
We say $f$ exhibits kernel behavior if the Linearization and Fixed Features properties hold for every component of $f$. The subsequent definition of a kernel analog also generalizes to a vector output, where $\nu_t$ is a vector in $\R^C$ and $\gK^{(\gA)}(\xi,\xi_t)$ is a matrix in $\R^{C\times C}$.} using $\theta$ as the parameters.
Gradient-based updates to the model parameters involve computing a loss function $\ell$ and $\frac{\partial \ell}{\partial\theta}$, which can be decomposed by the chain rule as $\frac{\partial\ell}{\partial f} \frac{\partial f}{\partial\theta}$. 
The first term is defined as the output derivative (\Cref{def:output_derivative}), and the second term is used to define kernel behavior (\Cref{def:kernel_regime}).

\begin{definition}[Output Derivative]\label{def:output_derivative}
    	The output derivative $\chi(\xi, y,\theta)$ for a network $f$ with parameters $\theta$, loss function $\ell$, and input $\xi$ with label $y$ is defined as $\chi(\xi, y, \theta) = \frac{\partial \ell(f(\xi;\theta), y)}{\partial f}$. We also define the output derivative applied at time $t$ as $\chi_t = \chi(\xi_t, y_t,  \theta_{t-1})$, where $\xi_t$ is the input at time $t$ with label $y_t$. For ease of notation, we often  absorb $y$ into $\xi$ and write $\chi(\xi, \theta)$ and $\chi(\xi, f)$ interchangeably. 
\end{definition}

Below, we adapt the definition of kernel-based learning (i.e., {\em lazy regime} in \citet{woodworth2019kernel}) to an arbitrary initialization. 
\begin{definition}[Kernel Behavior]
	Let $\theta_t$ be the parameters after $t$ steps of training by a gradient-based optimization algorithm, and let $\xi$ be an arbitrary fixed input. We say this training process of the network demonstrates \textit{kernel behavior} if the following properties are satisfied.
    \setlist{nolistsep}
	\begin{enumerate}[leftmargin=0.05\textwidth]
   		\item \textit{Linearization}: The change of the network can be well-approximated by its first order Taylor expansion, i.e., 
\[f(\xi; \theta_{t})-f(\xi; \theta_{t-1})\approx \langle \nabla f(\xi; \theta_{t-1}), \theta_{t} - \theta_{t-1}\rangle;\] 
		\item \textit{Fixed Features}: The gradient at step $t$ is approximately the same as before training, i.e.,
$\nabla f(\xi; \theta_t) \approx \nabla f(\xi; \theta_0)$.
	\end{enumerate}
	 $\nabla f$ denotes the gradient of $f$ w.r.t. $\theta$. ``Closeness to kernel behavior'' is quantified using the difference in the quantities on the two sides of the $\approx$ symbol. We formalize the approximations in \Cref{def:theory_kernel_behavior}.
	\label{def:kernel_regime}
\end{definition}
Past work has shown that if gradient-based training exhibits kernel behavior, then the function change can be expressed in terms of a fixed kernel (i.e., the kernel analog).
\begin{definition}[Kernel Analog]\label{def:kernel_analog}
    Suppose optimization of the parameters $\theta$ of a model $f$ using the gradient-based update algorithm $\gA$ exhibits kernel behavior (\Cref{def:kernel_regime}).
    Then, we say that a kernel $\gK^{(\gA)}$ is the \textit{kernel analog} of the optimization algorithm $\gA$ if for every $t>0$, there exists  $\nu_t$ such that for any input $\xi$,
    \begin{equation}\label{eq:kernel_gradient_descent}
        f(\xi;\theta_{t}) - f(\xi;\theta_{t-1}) \approx -\nu_t\gK^{(\gA)}(\xi,\xi_t)
    \end{equation} 
    where $\xi_t$ is the training input\footnote{For simplicity, we assume the batch size is 1.} of step $t$, $\theta_t$ is the parameter after step $t$.
\end{definition}
We illustrate the connection between the kernel analog and kernel behavior when using SGD.
If SGD exhibits kernel behavior, then, for a fixed input $\xi$, we can write
\begin{align*}
f(\xi; \theta_t)-f(\xi; \theta_{t-1})&\approx \langle \nabla f(\xi; \theta_{t-1}), \theta_t - \theta_{t-1}\rangle \\
&=\langle \nabla f(\xi; \theta_{t-1}), -\eta\chi_t\nabla f(\xi_t;\theta_{t-1})\rangle \\
& \approx-\eta\chi_t\ksgd(\xi,\xi_t)
\end{align*}
where the approximations follow from the Linearization and Fixed Features property, respectively, $\eta$ is the learning rate, $\chi_t$ is the output derivative (\Cref{def:output_derivative}), and $\ksgd$ is the kernel analog of SGD with $\nu_t=\eta\chi_t$.
Notably, $\ksgd$ is the well-known neural tangent kernel (NTK) formula derived in \cite{jacot2018neural}, which represents an input $\xi$ as the resulting gradient $\nabla f(\xi;\theta_0)$.
\begin{definition}[Neural Tangent Kernel $\ksgd$] 
	$\ksgd(\xi, \xi') = \langle \nabla f(\xi;\theta_0), \nabla f(\xi';\theta_0)\rangle$
\end{definition}

Given a kernel $\gK$, one can solve a classification task by learning $\alpha_i$ to minimize the empirical risk of $\sum_i\alpha_i \gK(\cdot,\xi_i)$, where $\{\xi_i\}$ is the training data (\Cref{sec:app_experiments}).
If training exhibits kernel behavior and $\gK$ is the kernel analog for the optimizer, then solving the kernel regression problem is equivalent to training the network \citep{jacot2018neural}.

In \Cref{sec:adam_kernel}, we derive the kernel analog for SignGD (i.e., an early-stage approximation of Adam), and in \Cref{sec:experiments}, we compare its eNTK performance against Adam FT.
The eNTK computation relies on two design choices for the setting: (1) what the model output $f(\xi;\theta)$ is, and (2) which optimizer $\gA$ is used.
We choose $f$ based on the FT setting (\Cref{sec:pt_ft_prelims}) and $\gA$ as SGD or Adam.

\section{Kernel Derivation for Adam}\label{sec:adam_kernel}

Computing the eNTK requires using the kernel analog (\Cref{def:kernel_analog}) of the chosen optimization algorithm $\gA$.
However, it is difficult to construct a long-term kernel analog for Adam, because the adaptivity causes each update to depend on the entire gradient history.
Previous work has shown that in the early stages of training, full-batch~\citep{ma2020qualitative} and mini-batch~\citep{malladi2022sdes} Adam with a small learning rate compute
the moving averages for the moment estimates in a small neighborhood, so the Adam update reduces to coordinate-wise normalization on the gradient. 
This optimization algorithm is called SignGD.
\begin{definition}[SignGD]
	\label{def:signgd} SignGD is a gradient-based optimization algorithm that updates parameters as
	$\theta_t = \theta_{t-1} - \eta \sign(\nabla \ell_t(\xi_t;\theta_{t-1}))$, 
	where $\sign$ is applied element-wise.
\end{definition}
In~\Cref{tab:signgd_ft}, we provide empirical evidence that fine-tuning with SignGD yields comparable performance to Adam.\footnote{Sign-based optimizers have also shown success in vision tasks~\citep{chen2022evolved}.}
We define the sign-based kernel below and prove it to be the correct kernel analog for SignGD.
\begin{definition}[Asymmetric SignGD Kernel]\label{def:asigngd_kernel}
	$\kasigngd(\xi, \xi') = \langle \nabla f(\xi;\theta_0), \sign(\nabla f(\xi';\theta_0)\rangle$.
\end{definition}
\begin{theorem}[Informal version of \Cref{thm:theory_signgd_kernel}] \label{thm:asigngd_kernel}
If a network is trained with SignGD and exhibits kernel behavior (\Cref{def:kernel_regime}), then the training dynamics follow
\[f(\xi; \theta_{t})-f(\xi; \theta_{t-1})\approx-\eta\sign(\chi_t)\kasigngd(\xi,\xi_t),\]
where $\chi_t$ is the output derivative (\Cref{def:output_derivative}).
\end{theorem}
\begin{proof}[Proof sketch]
The Linearization property in \Cref{def:kernel_regime} implies that
\begin{align*}
f(\xi; \theta_{t})-&f(\xi; \theta_{t-1}) \approx \langle \nabla f(\xi; \theta_t), \theta_{t} - \theta_{t-1}\rangle \\
&=-\eta\sign(\chi_t)\langle \nabla f(\xi; \theta_t), \sign(\nabla f(\xi_t;\theta_{t-1}))\rangle.
\end{align*}
Then, by the Fixed Features property in \Cref{def:kernel_regime}, 
\begin{align*}
\langle \nabla &f(\xi; \theta_t), \sign(\nabla f(\xi_t;\theta_{t-1}))\rangle\approx  \\  &\langle \nabla f(\xi; \theta_0), \sign(\nabla f(\xi_t;\theta_0))\rangle=\kasigngd(\xi,\xi_t).\qedhere 
\end{align*}
\end{proof}

We solve the asymmetric kernel regression as suggested in \citet{he2022learning}, %
but the difficulties of solving the kernel regression problem with an asymmetric kernel (\Cref{sec:app_asym_solver}) motivate us to also use the symmetric SignGD kernel.

\begin{definition}[SignGD Kernel]\label{def:signgd_kernel}
	$\ksigngd(\xi, \xi') = \langle \sign(\nabla f(\xi;\theta_0)), \sign(\nabla f(\xi';\theta_0))\rangle$
\end{definition}
Unlike the standard NTK formula for SGD, the kernel analog for Adam uses the sign function because early-stage Adam dynamics are agnostic to the scales of the gradients.
Concurrent work in~\citet{littwin2023adaptive} more formally extends the Tensor Programs framework and finds that no kernel can describe general (e.g., late-stage) Adam training when batch size is large.

\section{Theory: Prompt-Based Fine-Tuning Can Exhibit Kernel Behavior}\label{sec:theory}
We give a plausible mechanism for how prompt-based FT can exhibit kernel behavior (\Cref{def:kernel_regime}) as the network width grows large. 
We start by formalizing how changing the architecture width impacts pre-training.
\begin{definition}[Pre-Training Scheme]
 A pre-training scheme $(\gX, \gA, \gF^n)$ with width $n$ contains the dataset $\gX$, optimizer $\gA$ and its hyperparameters, and a model architecture $\gF^n$. Let $f^n\sim(\gX, \gA, \gF^n)$ denote a model resulting from training the architecture $\gF^n$ on the dataset $\gX$ with optimizer $\gA$.
\end{definition}
\begin{remark}
The reliance of the architecture on the width is given by Tensor Programs \citep{yang2020tensor2}: for example, in a Transformer, $n$ corresponds to the embedding dimension.
\end{remark}
We now connect pre-training to the downstream task. Analogous to~\citet{saunshi2020mathematical}, we reason that prompting transforms the downstream task into a fill-in-the-blank problem, and thus the downstream task can be viewed as a subcase of the pre-training task. We then assume that a wider pre-trained network will be better at filling in masked tokens and that an infinitely wide pre-trained network can solve the downstream task perfectly when using a suitable prompt.
\begin{definition}[Natural Task in the Infinite-Width Limit] \label{def:solvable_task}
    A downstream task $\Xi$ is natural with respect to a pre-training scheme $(\gX, \gA, \gF^n)$ if, for any pre-trained model $f^n\sim(\gX, \gA, \gF^n)$ and any downstream example $(\xi, y)\in\Xi$,
    \begin{equation}\label{eq:chi_zero}
	    \lim_{n\to\infty} \chi(\xi, y, f^n) = 0.
	\end{equation}
	where $\chi$ is the output derivative (\Cref{def:output_derivative}).
\end{definition}
\begin{remark}
Experiments in \Cref{sec:experiments} and \Cref{sec:app_prompt_choice} suggest that the FT optimization dynamics depend on the choice of prompt. 
In the above notation, the prompt is included in the downstream task dataset $\Xi$.
Only tasks with a well-suited prompt can be natural in the infinite-width limit. 
Tasks solved by FT using a randomly initialized head cannot satisfy the condition since  $\chi$ will not vanish even for an infinitely wide pre-trained network at start of FT.\end{remark}
Although \Cref{def:solvable_task} is asymptotic, we design a cheap empirical test using two models of different widths $n_1\neq n_2$ and same depth resulting from otherwise identical pre-training schemes: $f^{n_1}\sim (\gX, \gA, \gF^{n_1})$ and $f^{n_2}\sim (\gX, \gA, \gF^{n_2})$.
We measure if $\chi$ decreases with width for every downstream example $(\xi, y)\in\Xi$ without making any gradient updates.
This is necessary but not sufficient for the task to be natural in the infinite-width limit.
    See \Cref{sec:app_solvable_task}.

To study the behavior of FT, one also needs to make assumptions about parameters that resulted from pre-training. 
We assume that the network can be written as a Tensor Program~\citep{yang2019wide,yang2020tensor2,yang2020tensor3}, which is sufficiently general to allow our theory to describe many complex architectures (e.g., Transformers).
To allow the analysis to proceed by way of Tensor Programs, the network must be (1) \textit{stable}: its output does not grow with width (i.e., the infinite-width limit is meaningful), and (2) \textit{non-trivial}: its output can be updated during fine-tuning (i.e., learning can occur).

\begin{theorem}[Informal version of \Cref{thm:theory_prompt_finetuning}]\label{thm:prompt_ft_kernel_regime}
	Assume the downstream task $\Xi$ is natural in the infinite-width limit with respect to a pre-training scheme $(\gX, \gA, \gF^n)$, and the model $f\sim(\gX, \gA, \gF^n)$ is stable, non-trivial, and can be written as a Tensor Program.
	Then prompt-based FT of $f$ will exhibit the Linearization and Fixed Features properties of kernel behavior (\Cref{def:kernel_regime}).
\end{theorem}

The theorem formalizes the intuition that if the pre-trained network is already decent at solving the downstream task, the network needs to only mildly adapt to solve the downstream task. 
Notably, we extend standard NTK theory to account for an arbitrary initialization and to characterize early-stage training with Adam using results from~\Cref{sec:adam_kernel}.

Our theoretical results in this section and \Cref{sec:adam_kernel} apply to autoregressive and masked language models (MLMs), but we limit our fine-tuning experiments to MLMs as they are known to perform better after fine-tuning.

\section{Experiments}\label{sec:experiments}%

\definecolor{greenCircleGreen}{rgb}{0.09, 0.7, 0.27}
\definecolor{redSquareRed}{rgb}{0.9, 0.1, 0.1}
\newcommand{\cmark}{\raisebox{.5ex}{\tiny\textcolor{greenCircleGreen}{\newmoon}}}
\newcommand{\xmark}{\raisebox{.5ex}{\tiny\textcolor{redSquareRed}{\fullmoon}}}

\begin{table*}[t]
\centering
\resizebox{1.0\textwidth}{!}{
\centering
    \begin{tabular}{lcccccccccccccccccccc}
    \toprule
     Task  &  \multicolumn{1}{c}{\textbf{SST-2}} & \multicolumn{1}{c}{\textbf{SST-5}} & \multicolumn{1}{c}{\textbf{MR}} & \multicolumn{1}{c}{\textbf{CR}} & \multicolumn{1}{c}{\textbf{MPQA}} & \multicolumn{1}{c}{\textbf{Subj}} & \multicolumn{1}{c}{\textbf{TREC}} & \multicolumn{1}{c}{\textbf{AG News}} & \multicolumn{1}{c}{\textbf{MNLI}} & \multicolumn{1}{c}{\textbf{SNLI}} & \multicolumn{1}{c}{\textbf{QNLI}} & \multicolumn{1}{c}{\textbf{RTE}} & \multicolumn{1}{c}{\textbf{MRPC}} & \multicolumn{1}{c}{\textbf{QQP}} \\
    \midrule
    Task type & \multicolumn{4}{c}{------------ sentiment ------------} & \multicolumn{1}{c}{polarity} & \multicolumn{1}{c}{subj.} & \multicolumn{2}{c}{------ topic clf. ------} & \multicolumn{4}{c}{------------- entailment ------------} & \multicolumn{2}{c}{-- para. detect. --}\\
    Num.~classes $C$ & 2 & 5 & 2 & 2 & 2 & 2 & 6 & 4 & 3 & 3 & 2 & 2 & 2 & 2 \\
    \midrule
    \multicolumn{15}{c}{SGD vs. $\ksgd$: $16$-shot} \\
    \midrule
    eNTK solves task & \cmark\cmark\cmark\cmark\cmark & \cmark\cmark\cmark\cmark\cmark & \cmark\cmark\cmark\cmark\cmark & \cmark\cmark\cmark\cmark\cmark & \cmark\cmark\xmark\cmark\cmark & \cmark\cmark\cmark\cmark\cmark & \xmark\xmark\xmark\xmark\xmark & \cmark\cmark\cmark\cmark\cmark & \xmark\cmark\cmark\xmark\xmark & \cmark\xmark\xmark\xmark\xmark & \cmark\cmark\cmark\xmark\cmark & \cmark\cmark\cmark\cmark\cmark & \cmark\cmark\cmark\cmark\cmark & \xmark\cmark\cmark\cmark\cmark \\

    Linearization & \cmark\cmark\cmark\cmark\cmark & \xmark\cmark\cmark\cmark\cmark & \cmark\cmark\cmark\cmark\cmark & \cmark\cmark\cmark\cmark\cmark & \xmark\xmark\cmark\xmark\xmark & \cmark\cmark\cmark\cmark\cmark & \xmark\xmark\xmark\xmark\xmark & \cmark\cmark\cmark\cmark\cmark & \cmark\cmark\cmark\xmark\cmark & \cmark\cmark\cmark\xmark\xmark & \xmark\xmark\cmark\xmark\xmark & \cmark\cmark\xmark\cmark\cmark & \cmark\cmark\cmark\cmark\xmark & \cmark\cmark\cmark\xmark\xmark \\
    Fixed Features & \cmark\cmark\cmark\cmark\cmark & \cmark\cmark\cmark\cmark\cmark & \cmark\cmark\cmark\cmark\cmark & \cmark\cmark\cmark\cmark\cmark & \cmark\cmark\cmark\cmark\cmark & \cmark\cmark\cmark\cmark\cmark & \xmark\xmark\xmark\xmark\xmark & \cmark\cmark\cmark\cmark\xmark & \cmark\cmark\cmark\cmark\cmark & \cmark\cmark\cmark\cmark\cmark & \cmark\cmark\cmark\cmark\cmark & \cmark\cmark\cmark\cmark\cmark & \cmark\cmark\cmark\cmark\cmark & \cmark\cmark\cmark\cmark\cmark \\
    \midrule
    $\Rightarrow$ Kernel behavior & \cmark\cmark\cmark\cmark\cmark & \xmark\cmark\cmark\cmark\cmark & \cmark\cmark\cmark\cmark\cmark & \cmark\cmark\cmark\cmark\cmark & \xmark\xmark\xmark\xmark\xmark & \cmark\cmark\cmark\cmark\cmark & \xmark\xmark\xmark\xmark\xmark & \cmark\cmark\cmark\cmark\xmark & \xmark\cmark\cmark\xmark\xmark & \cmark\xmark\xmark\xmark\xmark & \xmark\xmark\cmark\xmark\xmark & \cmark\cmark\xmark\cmark\cmark & \cmark\cmark\cmark\cmark\xmark & \xmark\cmark\cmark\xmark\xmark \\

    \midrule
    \multicolumn{15}{c}{SGD vs. $\ksgd$: $64$-shot} \\
    \midrule
    eNTK solves task & \cmark\cmark\cmark\cmark\cmark & \cmark\cmark\cmark\cmark\cmark & \cmark\cmark\cmark\cmark\cmark & \cmark\cmark\cmark\cmark\cmark & \cmark\cmark\cmark\cmark\cmark & \cmark\cmark\cmark\cmark\cmark & \xmark\xmark\xmark\xmark\xmark & \cmark\cmark\cmark\cmark\cmark & \xmark\xmark\cmark\xmark\cmark & \xmark\xmark\xmark\xmark\xmark & \cmark\cmark\cmark\xmark\cmark & \cmark\cmark\cmark\cmark\cmark & \cmark\cmark\cmark\cmark\cmark & \cmark\cmark\cmark\cmark\cmark \\	
    Linearization & \cmark\cmark\cmark\cmark\cmark & \cmark\cmark\cmark\xmark\cmark & \cmark\cmark\cmark\cmark\cmark & \cmark\cmark\cmark\cmark\cmark & \xmark\xmark\xmark\cmark\cmark & \cmark\cmark\xmark\xmark\cmark & \xmark\xmark\xmark\xmark\xmark & \cmark\cmark\cmark\cmark\cmark & \xmark\xmark\xmark\xmark\xmark & \xmark\xmark\xmark\xmark\xmark & \xmark\cmark\xmark\cmark\xmark & \cmark\xmark\xmark\cmark\cmark & \cmark\xmark\cmark\xmark\xmark & \xmark\cmark\cmark\xmark\cmark \\						
    Fixed Features & \cmark\cmark\cmark\cmark\cmark & \cmark\cmark\cmark\cmark\cmark & \cmark\cmark\cmark\cmark\cmark & \cmark\cmark\cmark\cmark\cmark & \cmark\cmark\cmark\cmark\cmark & \cmark\cmark\cmark\cmark\cmark & \xmark\xmark\xmark\xmark\xmark & \cmark\cmark\cmark\xmark\cmark & \cmark\cmark\cmark\cmark\cmark & \cmark\cmark\cmark\cmark\cmark & \cmark\cmark\cmark\cmark\cmark & \cmark\cmark\cmark\cmark\cmark & \cmark\xmark\cmark\cmark\cmark & \cmark\cmark\cmark\cmark\cmark \\
    \midrule
    $\Rightarrow$ Kernel behavior & \cmark\cmark\cmark\cmark\cmark & \cmark\cmark\cmark\xmark\cmark & \cmark\cmark\cmark\cmark\cmark & \cmark\cmark\cmark\cmark\cmark & \xmark\xmark\xmark\cmark\cmark & \cmark\cmark\xmark\xmark\cmark & \xmark\xmark\xmark\xmark\xmark & \cmark\cmark\cmark\xmark\cmark & \xmark\xmark\xmark\xmark\xmark & \xmark\xmark\xmark\xmark\xmark & \xmark\cmark\xmark\xmark\xmark & \cmark\xmark\xmark\cmark\cmark & \cmark\xmark\cmark\xmark\xmark & \xmark\cmark\cmark\xmark\cmark \\

    \midrule
    \multicolumn{15}{c}{Adam vs. $\{\ksigngd, \kasigngd\}$: $16$-shot} \\
    \midrule
    eNTK solves task & \cmark\cmark\cmark\cmark\cmark & \cmark\cmark\cmark\xmark\cmark & \cmark\cmark\cmark\cmark\cmark & \cmark\cmark\cmark\cmark\cmark & \cmark\cmark\xmark\cmark\cmark & \cmark\cmark\cmark\cmark\cmark & \xmark\xmark\xmark\xmark\xmark & \cmark\cmark\cmark\cmark\cmark & \cmark\cmark\cmark\cmark\xmark & \xmark\xmark\cmark\xmark\xmark & \cmark\cmark\cmark\xmark\cmark & \cmark\cmark\cmark\cmark\cmark & \cmark\cmark\cmark\cmark\cmark & \cmark\xmark\xmark\cmark\cmark  \\		
    Linearization & \cmark\cmark\cmark\cmark\cmark & \cmark\cmark\xmark\cmark\cmark & \cmark\cmark\cmark\xmark\cmark & \cmark\cmark\cmark\cmark\cmark & \xmark\xmark\cmark\cmark\xmark & \cmark\cmark\cmark\cmark\cmark & \xmark\xmark\xmark\xmark\xmark & \xmark\cmark\xmark\xmark\xmark & \xmark\cmark\xmark\xmark\xmark & \xmark\cmark\xmark\xmark\xmark & \xmark\xmark\cmark\cmark\cmark & \xmark\cmark\cmark\cmark\cmark & \cmark\cmark\cmark\xmark\xmark & \cmark\cmark\cmark\xmark\cmark  \\		
    Fixed Features & \cmark\cmark\cmark\cmark\cmark & \cmark\cmark\cmark\cmark\cmark & \cmark\cmark\cmark\cmark\cmark & \cmark\cmark\cmark\cmark\cmark & \cmark\cmark\cmark\cmark\cmark & \cmark\cmark\cmark\cmark\cmark & \cmark\cmark\cmark\cmark\cmark & \cmark\xmark\cmark\cmark\cmark & \cmark\cmark\cmark\cmark\cmark & \cmark\cmark\cmark\cmark\cmark & \cmark\cmark\cmark\cmark\cmark & \cmark\cmark\cmark\cmark\cmark & \cmark\cmark\cmark\cmark\cmark & \cmark\cmark\cmark\cmark\cmark  \\
    \midrule
    $\Rightarrow$ Kernel behavior & \cmark\cmark\cmark\cmark\cmark & \cmark\cmark\xmark\xmark\cmark & \cmark\cmark\cmark\xmark\cmark & \cmark\cmark\cmark\cmark\cmark & \xmark\xmark\xmark\cmark\xmark & \cmark\cmark\cmark\cmark\cmark & \xmark\xmark\xmark\xmark\xmark & \xmark\xmark\xmark\xmark\xmark & \xmark\cmark\xmark\xmark\xmark & \xmark\xmark\xmark\xmark\xmark & \xmark\xmark\cmark\xmark\cmark & \xmark\xmark\cmark\cmark\cmark & \cmark\cmark\cmark\xmark\xmark & \cmark\xmark\xmark\xmark\cmark \\
    \midrule
    \multicolumn{15}{c}{Adam vs. $\{\ksigngd, \kasigngd\}$: $64$-shot} \\
    \midrule
    eNTK solves task & \cmark\cmark\cmark\cmark\cmark & \cmark\cmark\cmark\cmark\cmark & \cmark\cmark\cmark\cmark\cmark & \cmark\cmark\cmark\cmark\cmark & \cmark\cmark\cmark\cmark\cmark & \cmark\cmark\cmark\cmark\cmark & \xmark\xmark\xmark\cmark\xmark & \cmark\cmark\cmark\cmark\cmark & \xmark\cmark\xmark\cmark\cmark & \xmark\xmark\xmark\xmark\xmark & \xmark\xmark\cmark\xmark\xmark & \cmark\cmark\cmark\cmark\cmark & \cmark\cmark\cmark\cmark\cmark & \cmark\xmark\cmark\cmark\cmark  \\			
    Linearization & \cmark\cmark\cmark\cmark\cmark & \cmark\cmark\cmark\xmark\cmark & \cmark\cmark\cmark\cmark\cmark & \cmark\cmark\cmark\cmark\cmark & \xmark\xmark\xmark\cmark\cmark & \cmark\cmark\cmark\xmark\cmark & \xmark\xmark\xmark\xmark\xmark & \xmark\xmark\cmark\cmark\xmark & \xmark\cmark\cmark\xmark\xmark & \xmark\xmark\xmark\xmark\xmark & \xmark\cmark\cmark\cmark\xmark & \cmark\xmark\cmark\xmark\cmark & \cmark\cmark\cmark\cmark\cmark & \cmark\xmark\cmark\cmark\cmark \\
    Fixed Features & \cmark\cmark\cmark\cmark\cmark & \cmark\cmark\cmark\cmark\cmark & \cmark\cmark\cmark\cmark\cmark & \cmark\cmark\cmark\cmark\cmark & \cmark\cmark\cmark\cmark\cmark & \cmark\cmark\cmark\cmark\cmark & \cmark\cmark\cmark\cmark\cmark & \cmark\cmark\cmark\cmark\cmark & \cmark\cmark\cmark\cmark\cmark & \cmark\cmark\cmark\cmark\cmark & \cmark\cmark\cmark\cmark\cmark & \cmark\cmark\cmark\cmark\cmark & \cmark\cmark\cmark\cmark\cmark & \cmark\cmark\cmark\cmark\cmark \\
    \midrule
    $\Rightarrow$ Kernel behavior &
    \cmark\cmark\cmark\cmark\cmark & \cmark\cmark\cmark\xmark\cmark & \cmark\cmark\cmark\cmark\cmark & \cmark\cmark\cmark\cmark\cmark & \xmark\xmark\xmark\cmark\cmark & \cmark\cmark\cmark\xmark\cmark & \xmark\xmark\xmark\xmark\xmark & \xmark\xmark\cmark\cmark\xmark & \xmark\cmark\xmark\xmark\xmark & \xmark\xmark\xmark\xmark\xmark & \xmark\xmark\cmark\xmark\xmark & \cmark\xmark\cmark\xmark\cmark & \cmark\cmark\cmark\cmark\cmark & \cmark\xmark\cmark\xmark\cmark \\
    \bottomrule
    \end{tabular}
    }
    \caption{
    We find that $8$ out of $14$ tasks consistently induce kernel behavior across 5 subsampled datasets.
    Each dot represents a seed (i.e. a different $k$-shot dataset). A green dot indicates that the seed satisfies the criterion, and a red circle indicates that it does not. 
    We say the eNTK solves the task if the kernel analog achieves at least 90\% of the fine-tuning performance.
    We say that the Linearization property holds if the linearized model improves the pre-trained model by at least $50\%$ of the amount that fine-tuning improves it.
    We say that Fixed Features is satisfied if the average element-wise distance between the kernels before and after fine-tuning are less than $2.0$.
    The formal definition of kernel behavior (\Cref{def:kernel_regime}) does not prescribe measurable numerical thresholds for these properties, so we selected them manually for ease of presentation. 
    We urge readers to examine the data in \Cref{tab:main_prompted,fig:linearization_figure} directly for a more nuanced view. 
    } 
    \label{tab:summary}
\end{table*}

We compute the eNTK as described in \Cref{sec:prelims} for different optimization algorithms and FT settings.
eNTK performance being comparable to FT performance is a necessary but not sufficient condition for FT to exhibit kernel behavior (\Cref{def:kernel_regime}), so we also directly measure if the Linearization and Fixed Features properties hold (\Cref{sec:kernel_measurements}).
If the eNTK can solve the task, then eNTK regression provides an alternate method\footnote{The eNTK is not as susceptible to noisy gradients as FT is, because the learned kernel coefficients can downweight anomalous examples. 
This stability sometimes allows the kernel to outperform FT, especially in the few-shot setting (see MR and Subj in \Cref{tab:main_prompted_single_sentence}).} to use the pre-trained model to solve a downstream task, but the kernel lens only admits a theoretical analysis of FT optimization dynamics if both properties of kernel behavior are satisfied (\Cref{def:kernel_regime}; see \Cref{sec:kernel_measurements}).
For tasks that the eNTK cannot solve, we conjecture that the prompt is not well-designed for the task (in the sense of \Cref{def:solvable_task}), forcing the pre-trained model to adapt more during FT.

Our experiments are in the few-shot setting with manual prompt templates from \citet{gao2020making}. 
We consider 14 NLP tasks, divided into 8 single sentence and 6 sentence pair datasets, which cover: sentiment analysis (SST-2, SST-5, MR, CR); classifying an opinion's polarity (MQPA) or subjectivity (Subj) or question type (TREC) or news topic (AG News); natural language inference (MNLI, SNLI, QNLI, RTE); and paraphrase detection tasks (MRPC, QQP).
For each task, we randomly sample 5 $k$-shot datasets with $k$ training examples for each label.
We show experiments for $k \in \{16, 64\}$ using a pre-trained RoBERTa-base \citep{liu2019roberta} for all FT and eNTK experiments. 
We consider $\ksigngd$ and $\kasigngd$ as kernel analogs for Adam.
See \Cref{sec:app_experiments} for more details and experiments on $k=512$.

We summarize our results in \Cref{tab:summary}.
We find that the eNTK can consistently solve 12 out of 14 tasks comparably to prompt-based fine-tuning,
out of which 8 induce kernel behavior during fine-tuning.
Our results show that FT optimization dynamics depend on the downstream task and the inclusion of a meaningful prompt.

\begin{table*}[!h]
    \begin{subtable}[h]{1.0\textwidth}
        \centering

        \resizebox{0.87\textwidth}{!}{  
        
        \begin{tabular}{rlcccccccccccccc}
        \toprule
        $k$-shot & Method & {SST-2} & {SST-5} & {MR} & {CR} & {MPQA} & {Subj} & {TREC} & {AG News} \\
        \midrule
            \multirow[t]{5}{*}{16} & SGD-FT 
                & \tf{89.0}$_{(1.5)}$ & \tf{44.6}$_{(1.4)}$ & 83.2$_{(2.4)}$ & \tf{93.3}$_{(0.2)}$ & \tf{83.3}$_{(1.3)}$ & 88.5$_{(2.6)}$ & \tf{80.3}$_{(7.2)}$ & \tf{84.2}$_{(1.1)}$ \\
            & {\cellcolor{gblue!10}$\ksgd$} 
                & {\cellcolor{gblue!10}88.3$_{(0.3)}$} & {\cellcolor{gblue!10}43.6$_{(2.2)}$} & {\cellcolor{gblue!10}\tf{84.7}$_{(1.5)}$} & 	{\cellcolor{gblue!10}93.2$_{(0.9)}$} & {\cellcolor{gblue!10}76.4$_{(2.7)}$} & {\cellcolor{gblue!10}\tf{88.6}$_{(1.3)}$} & {\cellcolor{gblue!10}56.0$_{(9.2)}$} & {\cellcolor{gblue!10}82.1$_{(2.0)}$} \\
        \cmidrule{2-10} 
            & Adam-FT 
                & \tf{88.3}$_{(1.2)}$ & \tf{45.4}$_{(2.6)}$ & 81.3$_{(6.1)}$ & 93.0$_{(1.6)}$ & \tf{82.8}$_{(2.2)}$ & 87.4$_{(2.1)}$ & \tf{79.6}$_{(6.1)}$ & \tf{84.0}$_{(1.6)}$ \\
            & {\cellcolor{ggreen!7}$\ksigngd$} 
                & {\cellcolor{ggreen!7}\tf{88.3}$_{(0.5)}$} & {\cellcolor{ggreen!7}42.2$_{(3.9)}$} & {\cellcolor{ggreen!7}84.3$_{(1.5)}$} & {\cellcolor{ggreen!7}\textbf{93.7}$_{(0.5)}$} & {\cellcolor{ggreen!7}76.7$_{(3.3)}$} & {\cellcolor{ggreen!7}\tf{89.2}$_{(2.0)}$} & {\cellcolor{ggreen!7}58.1$_{(6.5)}$} & {\cellcolor{ggreen!7}82.3$_{(1.6)}$} \\
            & {\cellcolor{ggreen!7}$\kasigngd$}
                & {\cellcolor{ggreen!7}\tf{88.3}$_{(0.4)}$} & {\cellcolor{ggreen!7}43.7$_{(1.7)}$} & {\cellcolor{ggreen!7}\tf{84.9}$_{(1.1)}$} & 	{\cellcolor{ggreen!7}93.4$_{(0.5)}$} & {\cellcolor{ggreen!7}74.6$_{(3.5)}$} & {\cellcolor{ggreen!7}88.6$_{(1.8)}$} & {\cellcolor{ggreen!7}22.7${{}}_{(2.8)}$} & {\cellcolor{ggreen!7}83.6$_{(1.0)}$}  \\
        \cmidrule{1-10} 
            \multirow[t]{5}{*}{64} & SGD-FT 
                & \tf{89.7}$_{(0.4)}$ & 45.8$_{(2.1)}$ & 85.6$_{(1.1)}$ & \tf{94.3}$_{(0.5)}$ & \tf{84.8}$_{(0.8)}$ & \tf{92.9}$_{(0.5)}$ & \tf{93.2}$_{(1.0)}$ & \tf{86.8}$_{(0.7)}$ \\
            & {\cellcolor{gblue!10}$\ksgd$} 
                & {\cellcolor{gblue!10}89.2$_{(1.0)}$} & {\cellcolor{gblue!10}\tf{46.0}$_{(1.3)}$} & {\cellcolor{gblue!10}\tf{86.4}$_{(0.6)}$} & {\cellcolor{gblue!10}93.7$_{(0.4)}$} & {\cellcolor{gblue!10}81.2$_{(0.9)}$} & {\cellcolor{gblue!10}91.4$_{(0.7)}$} & {\cellcolor{gblue!10}77.8$_{(2.3)}$} & {\cellcolor{gblue!10}85.6$_{(0.7)}$} \\
        \cmidrule{2-10} 
            & Adam-FT 
                & \tf{89.3}$_{(0.7)}$ & 48.5$_{(2.0)}$ & \tf{86.0}$_{(0.4)}$ & {93.7}$_{(0.8)}$ & \tf{84.6}$_{(0.9)}$ & \tf{92.7}$_{(0.6)}$ & \tf{92.6}$_{(1.3)}$ & \tf{86.8}$_{(1.1)}$ \\
            & {\cellcolor{ggreen!7}$\ksigngd$}
                & {\cellcolor{ggreen!7}89.1$_{(0.5)}$} & {\cellcolor{ggreen!7}\tf{49.1$_{(1.6)}$}} & {\cellcolor{ggreen!7}85.6$_{(1.0)}$} &
                {\cellcolor{ggreen!7}93.9$_{(0.2)}$} &
                {\cellcolor{ggreen!7}79.0$_{(5.8)}$} & {\cellcolor{ggreen!7}92.4$_{(0.5)}$} & {\cellcolor{ggreen!7}82.0$_{(1.4)}$} & {\cellcolor{ggreen!7}85.9$_{(0.7)}$} \\
            & {\cellcolor{ggreen!7}$\kasigngd$}
                & {\cellcolor{ggreen!7}88.9$_{(0.9)}$} &  {\cellcolor{ggreen!7}43.6$_{(2.2)}$} & {\cellcolor{ggreen!7}85.6$_{(1.0)}$} & {\cellcolor{ggreen!7}\tf{94.0}$_{(0.3)}$} & {\cellcolor{ggreen!7}81.8$_{(1.1)}$} & {\cellcolor{ggreen!7}91.8$_{(1.1)}$} &  	{\cellcolor{ggreen!7}21.0$_{(4.3)}$} & {\cellcolor{ggreen!7}86.2$_{(0.3)}$} \\
        \bottomrule
        \end{tabular}

        }
        
        \caption{Single-sentence tasks \vspace{1.0em}}
        \label{tab:main_prompted_single_sentence}
    \end{subtable}
    \begin{subtable}[h]{1.0\textwidth}
        \centering

        \resizebox{.7\textwidth}{!}{  
\begin{tabular}{rlcccccccccccc}
\toprule
$k$-shot & Method & {MNLI} & {SNLI} & {QNLI} & {RTE} & {MRPC} & {QQP} \\
\midrule
    \multirow[t]{5}{*}{16} & SGD-FT
        & \tf{59.2}$_{(2.7)}$ & \tf{65.7}$_{(2.7)}$ & \tf{62.1}$_{(3.1)}$ & \tf{60.0}$_{(5.5)}$ & \tf{73.9}$_{(2.7)}$ & \tf{62.1}$_{(2.3)}$ \\
    & {\cellcolor{gblue!10}$\ksgd$}
        & {\cellcolor{gblue!10}53.0$_{(3.0)}$} & {\cellcolor{gblue!10}57.8$_{(2.3)}$} & {\cellcolor{gblue!10}60.1$_{(3.3)}$} & {\cellcolor{gblue!10}\tf{60.0}$_{(4.7)}$} & {\cellcolor{gblue!10}73.4$_{(5.6)}$} & {\cellcolor{gblue!10}58.2$_{(0.9)}$} \\
\cmidrule{2-8} 
    & Adam-FT
        & \tf{56.8}$_{(2.9)}$ & \tf{64.6}$_{(4.1)}$ & \tf{63.1}$_{(3.5)}$ & 57.6$_{(6.3)}$ & \tf{77.6}$_{(3.1)}$ & \tf{61.8}$_{(4.5)}$ \\
    & {\cellcolor{ggreen!7}$\ksigngd$}
        & {\cellcolor{ggreen!7}53.8$_{(1.2)}$} & {\cellcolor{ggreen!7}54.9$_{(2.7)}$} %
        & {\cellcolor{ggreen!7}59.5$_{(3.1)}$} & {\cellcolor{ggreen!7}55.4$_{(4.2)}$} & {\cellcolor{ggreen!7}75.6$_{(1.2)}$} & {\cellcolor{ggreen!7}60.7$_{(2.2)}$} \\
    & {\cellcolor{ggreen!7}$\kasigngd$}
        & {\cellcolor{ggreen!7}51.9$_{(4.0)}$} & {\cellcolor{ggreen!7}54.9$_{(3.1)}$} & {\cellcolor{ggreen!7}56.0$_{(1.9)}$} & {\cellcolor{ggreen!7}\tf{59.8}$_{(4.0)}$} & {\cellcolor{ggreen!7}75.2$_{(2.6)}$} & {\cellcolor{ggreen!7}59.4$_{(2.0)}$} \\
\cmidrule{1-8} 
    \multirow[t]{5}{*}{64} & SGD-FT
        & \tf{68.7}$_{(1.7)}$ & \tf{77.3}$_{(0.9)}$ & \tf{72.8}$_{(2.2)}$ & \tf{68.9}$_{(2.5)}$ & \tf{82.8}$_{(1.2)}$ & \tf{69.2}$_{(1.3)}$ \\
    & {\cellcolor{gblue!10}$\ksgd$}
        & {\cellcolor{gblue!10}60.4$_{(1.8)}$} & {\cellcolor{gblue!10}65.5$_{(1.6)}$} & {\cellcolor{gblue!10}67.3$_{(1.6)}$} & {\cellcolor{gblue!10}66.5$_{(2.5)}$} & {\cellcolor{gblue!10}79.2$_{(2.5)}$} & {\cellcolor{gblue!10}66.4$_{(1.7)}$} \\
\cmidrule{2-8} 
    & Adam-FT
        & \tf{67.9}$_{(1.0)}$ & \tf{76.9}$_{(1.4)}$ & \tf{74.2}$_{(3.2)}$ & \tf{67.3}$_{(2.7)}$ & \tf{80.9}$_{(1.2)}$ & \tf{69.8}$_{(0.6)}$ \\
    & {\cellcolor{ggreen!7}$\ksigngd$}
        & {\cellcolor{ggreen!7}60.8$_{(1.7)}$} %
        & {\cellcolor{ggreen!7}64.1$_{(2.3)}$} %
        & {\cellcolor{ggreen!7}65.4$_{(1.7)}$} & {\cellcolor{ggreen!7}63.8$_{(1.8)}$} & {\cellcolor{ggreen!7}77.4$_{(2.3)}$} & {\cellcolor{ggreen!7}63.7$_{(4.4)}$} \\
    & {\cellcolor{ggreen!7}$\kasigngd$}
        & {\cellcolor{ggreen!7}58.5$_{(1.7)}$} & {\cellcolor{ggreen!7}66.8$_{(1.1)}$} & {\cellcolor{ggreen!7}66.5$_{(1.1)}$} & {\cellcolor{ggreen!7}63.8$_{(2.2)}$} & {\cellcolor{ggreen!7}77.3$_{(2.0)}$} & {\cellcolor{ggreen!7}66.1$_{(3.4)}$} \\
\bottomrule
\end{tabular}

        }
        
        \caption{Sentence-pair tasks \vspace{-0.5em}}        \label{tab:main_prompted_sentence_pair}
    \end{subtable}

\caption{Prompt-based FT and prompt-based eNTK performance with different formulas on the LM-BFF test set \citep{gao2020making}. 
The kernel analog performs comparably to FT on many tasks but fails if the prompt is poorly designed (i.e., MPQA, TREC, SNLI, and MNLI). Performance is measure by average test accuracy over 5 $k$-shot splits for all tasks except MRPC and QQP, where it is F1.
}

\label{tab:main_prompted}
\end{table*}

\subsection{Kernel Performance on Downstream Tasks}

\paragraph{Prompting is critical for eNTK to match FT performance.}
We measure the eNTK performance in the standard and prompt-based FT settings across SST-2, MR, CR, QNLI, QQP and RTE (\Cref{fig:prompt_vs_noprompt} and \Cref{tab:noprompt}).
In the standard FT setting, $\ksgd$ and SGD-FT demonstrate a gap of up to $16\%$ absolute on tasks that exhibit only a 3\% gap in the prompt-based setting. 
\Cref{tab:noprompt} demonstrates that the inclusion of more data improves the eNTK performance in the unprompted setting, but kernels computed with a prompt consistently outperform the standard ones.
We explore the importance of the choice of prompt format in \Cref{sec:app_prompt_choice}. 
These results agree with our theoretical analysis that tasks must use a meaningful prompt in order to induce kernel behavior (\Cref{def:solvable_task}).

\begin{figure*}[t]
    \centering
    \includegraphics[width=\textwidth]{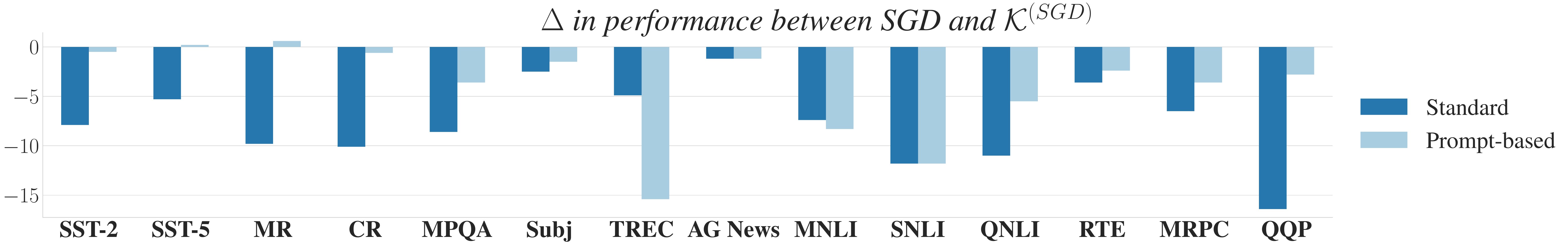}
    \caption{
    The performance difference between SGD-FT and $\ksgd$ performance for both the standard 
    and the prompt-based setting
    (\Cref{sec:prelims}) suggests that using a prompt is important for kernel behavior (\Cref{def:kernel_regime}) to arise. 
    In standard FT, we initialize the new classification head (i.e., $\Gamma$) using the linear probing solution.
    The performance is shown for the $64$-shot setting and measured by the average test accuracy over 5 random splits, except for MRPC and QQP, where it is F1.
    Results on additional settings are in \Cref{tab:noprompt}.
    }
    \label{fig:prompt_vs_noprompt}
\end{figure*}

\paragraph{SGD performs comparably to Adam in prompt-based FT.} \Cref{tab:main_prompted} shows that Adam and SGD perform within 4\% absolute of each other when using a prompt, 
suggesting that known difficulties in optimizing transformers with SGD \citep{li2022robust,zhang2019why,liu2020understanding} do not play a substantial role during prompt-based FT.
Indeed, we expect that the benefit of Adam over SGD is reduced when the task is simple enough to induce kernel behavior.

\paragraph{Prompt-based eNTK matches FT in most tasks.}
We compare SGD-FT to $\ksgd$ and Adam-FT to $\kasigngd$ in \Cref{tab:main_prompted}. 
We observe that for 10 out of 14 tasks, the kernel analog can achieve accuracy within $10\%$ of the corresponding FT performance for $k=16$ and $k=64$. 
The difference between $\ksigngd$ and $\kasigngd$ is negligible on most tasks, but the non-standard solver for the asymmetric problem (\Cref{sec:app_asym_solver}) may cause $\kasigngd$ to sometimes perform worse than $\ksigngd$ despite being the theoretically sound kernel analog (\Cref{thm:asigngd_kernel}).

\subsection{Measuring Kernel Behavior}\label{sec:kernel_measurements}
The eNTK can often solve the task comparably to fine-tuning (\Cref{tab:main_prompted}), suggesting that these tasks may induce kernel behavior (\Cref{def:kernel_regime}).
However, the observed success only indicates that the gradient features can solve the downstream task and does not directly study the optimization dynamics.
We take additional measurements to provide further empirical evidence that FT can be \textit{described} by kernel behavior.
The approximations in \Cref{def:kernel_regime} involve constants depending on the dataset and model architecture, so we set manual thresholds for our results.

\paragraph{The Linearization property holds for all tasks the eNTK can solve.} 
If FT exhibits kernel behavior (\Cref{def:kernel_regime}), then the function after FT should be close to the first order Taylor expansion around the pre-trained model:
$$f(\xi;\theta_{\text{FT}}) \approx f(\xi;\theta_{\text{PT}}) + \langle \nabla f(\xi; \theta_{\text{PT}}), \theta_{\text{FT}} - \theta_{\text{PT}} \rangle$$
where $\theta_{\text{PT}}$ is the model parameters after pre-training, $\theta_{\text{FT}}$ is the model parameters after fine-tuning on the downstream task, and $\xi$ is sampled from the test set.
\Cref{fig:linearization_figure} summarizes how this linearized model performs in comparison to the pre-trained and fine-tuned models.

Pre-trained models perform significantly better than random on many single-sentence downstream tasks (e.g., SST-2, MR, and CR) but close to random on most sentence-pair tasks (e.g., QNLI, RTE, MRPC, and QQP).\footnote{Subj, MNLI, and SNLI are outliers to this trend.} 
The linearized model recovers more than $50\%$ amount of the improvement from FT for all tasks the eNTK could solve (\Cref{tab:main_prompted}). 

\paragraph{The Fixed Features property holds for all tasks the eNTK can solve.}
We empirically test if the Fixed Features property (\Cref{def:kernel_regime}) holds for tasks that the eNTK can solve by measuring the relative distance between $\ksgd$ computed before and after FT
(\Cref{tab:linearization_kernel_distance}).
Tasks that the eNTK can solve exhibit low (i.e., less than 2) distances, indicating the Fixed Features property likely holds.

\paragraph{Entailment tasks exhibit anomalous optimization characteristics.}
Although pre-trained models perform much better than random on MNLI and SNLI, we find that the eNTK cannot solve these tasks very well (\Cref{tab:main_prompted} and \Cref{fig:linearization_figure}). 
Similarly, although the pre-trained model demonstrates near-random performance on QNLI and RTE, we find that the eNTK can solve these tasks. 
Moreover, although QNLI and RTE could sometimes be solved by the eNTK, the results suggest they do not induce the Linearization property of kernel behavior very strongly.
Altogether, these findings suggest a deeper mystery around the fine-tuning dynamics when solving entailment tasks.

\subsection{Tasks without Kernel Behavior}\label{sec:tasks_without_kernel}
TREC, MNLI, SNLI, QNLI, and MPQA consistently do not induce kernel behavior (\Cref{tab:summary}).\footnote{The eNTK can consistently solve AG News although \mbox{Adam-FT} does not exhibit kernel behavior.
This finding suggests that our theory holds for the prompt used with AG News, but the grid search over learning rates results in FT that does not exhibit kernel behavior.
In particular, the success of the eNTK suggests the task can be solved with a very small learning rate, but the FT trajectory achieving the best performance uses a larger learning rate and thus exhibits more complex dynamics.
}
Our theoretical analysis suggests that when the prompt and label words do not format the task as a subcase of pre-training, then the task will not be natural in the infinite-width limit (\Cref{def:solvable_task}) and hence will not induce kernel behavior. 

Considering the prompt templates shown in \Cref{sec:app_datasets},
we suggest that the TREC prompt (simply a colon) provides insufficient signal to the model to perform question type classification.
For MNLI and SNLI, we observe that connecting the premise and hypothesis with the label word ``Maybe'' for neutral examples results in ungrammatical sentences.
Analogously, for QNLI, we note that the premise is often a question without a clear yes or no answer, so the label words are unnatural to place between the premise and hypothesis.
The prompt used for sentiment and polarity tasks is designed to follow a complete sentence or substantial phrase, so it is less natural when used with MPQA examples, which are often only one or two words.
See \Cref{sec:app_prompt_choice} for prompt ablations.

\section{Efficacy of Subspace-Based Fine-Tuning Methods}\label{sec:parameter_efficient}
We study subspace-based fine-tuning methods, which apply updates to only a low-dimensional subspace of the high-dimensional model parameter space during fine-tuning.
Although theoretical analysis of these methods seems complex, the kernel view admits a simple interpretation.
We directly apply the Johnson-Lindenstrauss (JL) lemma in \citet{johnson1984extensions}, which guarantees inner product preservation under random projections,
to suggest why LoRA~\citep{hu2021lora} works. 
Similar analysis yields results on parameter-subspace FT methods used to study intrinsic dimension (\citet{li2018measuring,aghajanyan2021intrinsic}, see \Cref{sec:app_theory_lora}).
\begin{definition}[$\gA$-LoRA FT \citep{hu2021lora}] \label{def:lora_ft}
Let $\gA$ be a gradient-based optimization algorithm.
For every weight matrix $W\in\sR^{m\times n}$, choose $k\ll m$ and initialize $B\in\sR^{m\times k}$ with i.i.d. zero-mean Gaussian values and $A\in\sR^{k\times n}$ as 0. Set the weight to be $W + BA$. To fine-tune, fix $W$ at its pre-trained value and train only $A$ and $B$ using $\gA$.
\end{definition}

We show that if SGD FT exhibits kernel behavior, then so does SGD-LoRA FT, and SGD-LoRA FT using a sufficiently large $k$ does not modify the kernel or the dynamics.

\begin{theorem}[Informal version of \Cref{thm:theory_lora}]
    Let $\ksgd$ be the kernel analog (\Cref{def:kernel_analog}) to SGD FT and $\ksgdlora$ be the kernel analog to SGD-LoRA FT on a downstream task $\Xi$ with $N$ examples.
	Then, with high probabililty, $\ksgdlora(i,j)\approx \ksgd(i,j)$ 
 for all $i,j\in [N]$. 
	\label{thm:lora}
\end{theorem}
\begin{proof}[Proof sketch]
Consider an individual layer in the network and a task input $\xi\in\Xi$.
LoRA causes $\nabla_B f(\xi;\theta)$ to be a random projection of $\nabla_W f(\xi;\theta)$, 
where $\nabla_B$ denotes the gradient with respect to $B$, and $\nabla_A f(\xi;\theta)=0$ since $B$ is initialized to zero. 
The rest of the proof follows from applying JL to all input pairs $\xi,\xi'$ to show the inner product (and thus, the kernel entry) is preserved. 
\end{proof}
\vspace{-4pt}
\begin{remark}
	\Cref{thm:lora} states that the kernel analog of SGD-FT is unchanged by LoRA in both prompt-based and standard FT.
	However, 
 the theorem only provides an explanation for the success of $\gA$-LoRA FT when $\gA$ FT exhibits kernel behavior.
    Therefore, as per~\Cref{sec:theory,sec:experiments}, we consider this theorem to only be meaningful when considering prompt-based SGD and prompt-based LoRA-SGD.	
\end{remark}
\vspace{-4pt}
\Cref{tab:lora} verifies that prompt-based SGD FT and SGD-LoRA FT achieve similar performance on several tasks, and $\ksgdlora$ achieves performance similar to $\ksgd$.

\section{Conclusion}
We use NTKs to mathematically formalize the general intuition that fine-tuning pretrained language models to solve downstream tasks requires only a ``small change.''  
We derive a new kernel to describe Adam training (\Cref{sec:adam_kernel}) and we use it in \Cref{sec:theory} to show how prompt-based fine-tuning can exhibit kernel behavior.
Extensive experiments in \Cref{sec:experiments} on 14 NLU tasks demonstrate that including a meaningful prompt often causes FT to exhibit kernel behavior (\Cref{fig:prompt_vs_noprompt}) and that kernel dynamics \textit{describe} prompt-based FT on tasks that the eNTK can solve (\Cref{sec:kernel_measurements}).
We demonstrate one possible use of the kernel view to explain empirical phenomena by applying it to understand subspace-based fine-tuning methods (\Cref{sec:parameter_efficient}), and we note that the kernel has many mathematically useful properties that can aid design and study of alternate fine-tuning methods.

Our work suggests that a kernel-based view of language model fine-tuning is plausible, but there are several limitations.
First, our experiments are limited to few-shot classification tasks and a single masked language model with specific prompts.
Extending to additional settings (e.g., increasing $k$) and models require significant computational cost because the eNTK is expensive to compute.
The theoretical results also apply only to ``early-stage'' training with Adam, and it is not clear how well they can describe longer training schemes; concurrent work in~\citet{littwin2023adaptive} suggests that the reduction of Adam to SignGD is crucial to observe kernel dynamics.
Nevertheless, our work provides substantial empirical and theoretical evidence that  fine-tuning can be analyzed in terms of kernel behavior.

As a future direction, one can use the kernel analog to study the inductive bias of FT, as was done for gradient descent from a random initialization in the past \citep{allenzhu2018convergence,allenzhu2018learning,li2018learning}.
For example, several works~\citep{cao2019generalization,arora2019finegrained,wei2022more} have shown that the spectrum of the kernel can bound the generalization ability of the trained network, giving insight into why few-shot fine-tuning does not result in catastrophic overfitting.
Our experiments show some tasks do not induce kernel behavior during FT, suggesting that future theoretical analysis of FT needs to account for the downstream task and choice of prompt.

\section*{Acknowledgements}
We thank Tianyu Gao, Wei Hu, Jason Lee, Kaifeng Lyu, Abhishek Panigrahi, Nikunj Saunshi, Mengzhou Xia, and Greg Yang for their helpful comments and discussion.
This work is funded by NSF, ONR, Simons Foundation, DARPA and SRC. 
  
\newpage
\bibliography{bibliography}
\bibliographystyle{icml2023}

\newpage
\appendix
\onecolumn
\section{Experimental Details}\label{sec:app_experiments}

\subsection{Datasets and Prompts}
\label{sec:app_datasets}

\begin{table*}[h]
\centering
\resizebox{1\columnwidth}{!}{%
\begin{tabular}{llrrrll}
\toprule
 \tf{Dataset} & $C$ & \#\tf{Train} & \#\tf{Test} & \tf{Type} & \tf{Prompt} & \tf{Label words} \\
\bottomrule
 SST-2 & 2 & 67,349 & 872 & sentiment & {\sent} It was {\mask} . & \{great, terrible\} \\
 SST-5 & 5 & 8,544 & 1,000 & sentiment & {\sent} It was {\mask} . & \{great, good, okay, bad, terrible\} \\
 MR & 2  & 8,662& 1,000 & sentiment  & {\sent} It was {\mask} . & \{great, terrible\} \\
  CR & 2  & 3,175 & 500 & sentiment & {\sent} It was {\mask} . & \{great, terrible\}\\
 MPQA & 2  & 8,606 & 1,000 & opinion polarity & {\sent} It was {\mask} . & \{great, terrible\}\\
 Subj & 2  & 8,000 & 1,000 & subjectivity & {\sent} This is {\mask} . & \{subjective, objective\} \\
 TREC & 6  & 5,452 & 500 & question cls. & {\mask} : {\sent} & \{Description, Expression, Entity, \\
 & & & & & & Human, Location, Number\}\\
 AG News & 4 & 120,000 & 7,600 & news topic & {\sent} This article is about {\mask} news. & \{world, sports, business, tech\} \\

\midrule
  MNLI & 3  & 392,702 & 1,000 & NLI & {\firstsent} ? {\mask} , {\secondsent}  & \{Yes, Maybe, No\}\\
SNLI & 3  &  549,367 & 1,000 & NLI  & {\firstsent} ? {\mask} , {\secondsent} & \{Yes, Maybe, No\} \\
 QNLI & 2  & 104,743 & 1,000  & NLI & {\firstsent} ? {\mask} , {\secondsent} & \{Yes, No\} \\
 RTE & 2 & 2,490 & 277 & NLI & {\firstsent} ? {\mask} , {\secondsent} & \{Yes, No\}  \\
MRPC & 2   & 3,668 & 408 & paraphrase & {\firstsent} {\mask} , {\secondsent} & \{Yes, No\} \\
 QQP & 2 & 363,846 & 1,000 & paraphrase & {\firstsent} {\mask} , {\secondsent} & \{Yes, No\}  \\
\bottomrule
\end{tabular}
}
\caption{The statistics and prompts of the datasets we used in our experiments. The choices of prompts are adapted from \cite{gao2020making} and include a template and a set of label words that can fill in the \mask token. {\firstsent} and {\secondsent} refer to the first and the second (if any) input sentence.}
\label{tab:dataset_statistics}
\end{table*}

Table \ref{tab:dataset_statistics} shows the set of downstream tasks, which are adapted from \cite{gao2020making}.
We consider 8 single sentence classification datasets (SST-2~\citep{socher2013recursive_sst-2}, SST-5~\citep{socher2013recursive_sst-2}, MR~\citep{pang2005seeing_mr}, CR~\citep{hu2004mining_cr}, MPQA~\citep{wiebe2005annotating_mpqa}, Subj~\citep{pang2004sentimental_subj}, TREC~\citep{voorhees2000building_trec}, and AG News~\citep{zhang2015character_ag_news}), and 6 sentence pair datasets (MNLI~\citep{williams2018broad_mnli}, SNLI~\citep{bowman2015large}, QNLI~\citep{rajpurkar2016squad}, RTE~\citep{dagan2005pascal_rte1,bar2006second,giampiccolo2007third_rte3,bentivogli2009fifth_rte4}, MRPC~\citep{dolan2005automatically_mrpc} and QQP\footnote{\url{https://www.quora.com/q/quoradata/}}.
Our datasets represent 6/8 datasets of the GLUE benchmark \cite{wang2019glue} (SST-2, MNLI, QNLI, RTE, MRPC, QQP). 

In contrast to \cite{gao2020making}, we add AG News as an additional multi-label classification task, and make two modifications to the test sets. First, we split CR into 500 test examples and 3,175 training examples to ensure enough training examples for our $512$-shot experiments and secondly, we limit the test sizes to 1,000 examples to speed up kernel evaluations.

To generate $k$-shot few-shot datasets, the original training data is used to randomly sample $k$ examples per label for training and another, separate $k$ examples per label for the validation set. Unless otherwise stated, we usually run experiments over 5 seeds of few-shot data sets.
We directly use the  `manual' prompt templates and label words proposed by \cite{gao2020making}, which are reproduced in Table~\ref{tab:dataset_statistics}. We do include any demonstrations in our prompts.

\subsection{Computing the Kernel}
We use functorch \citep{functorch2021} to compute the eNTK for RoBERTa-base (125M parameters), using a mix of backward-mode auto-differentiation for computing the jacobians and forward-mode auto-differentiation for computing jacobian-vector products \citep{novak2022fast} .
Note that $\ksigngd$ cannot be computed via jacobian-vector products and requires substantially more memory and run-time in practice.

\subsection{Solving the Kernel}\label{sec:app_asym_solver}
In the standard NTK setting, the initial output of the model $f(\cdot;\theta_0)$ contains no information about solving the task, because $\theta_0$ is a random initiaization.
However, in the prompted FT setting, we expect the pre-trained model to be able to solve the downstream task well even before any fine-tuning occurs (see \Cref{tab:more_linearization_measurement}).
So, we add the pre-trained model's output to the output from the kernel.
Furthermore, we run a grid search over scaling the labels in order to take advantage of any pre-existing knowledge the model has about the downstream task. In particular, the kernel regression is based on the $\ell_2$ distance to the ground truth one-hot vector, but the pre-trained model outputs the logits which will be used for cross-entropy loss. Scaling the one-hot vector by $f_0$ helps align its scaling with the logits. Our hyperparameter grid for $f_0$ can be found in Table \ref{tab:hyperparameters}, where $\infty$ corresponds to not using the pre-trained model logits when solving the kernel.

\paragraph{Solving Multi-Class Tasks} There are several options for how to solve $C$-way classification tasks ($C> 2$). We perform the most general one, which scales with $C^2$. Each logit is treated as an independent output of the network, essentially scaling the size $N$ of the original dataset by a factor of $C$. With $CN$ examples, the kernel now has shape $CN\times CN$. The labels are also scaled up to treat the multi-class problem as many binary classification problems. Solving the multi-class task this way allows the kernel regression model to view relationships between different logits. 

\paragraph{Symmetric Kernel}
Given a symmetric kernel $\gK\in\sR^{N\times N}$, we solve the kernel regression problem. In particular, we use the representer theorem to write that the empirical risk minimizer of the loss can be expressed as a linear combination of the kernel features computed on the train set.
$$ h^*(\cdot) = \argmin_{h\in\gH_\gK} \frac{1}{N}\sum_{i=1}^N	\ell(h(x_i), y_i) \quad \leftrightarrow \quad h^*(\cdot) = \sum_{i=1}^N \alpha_i \gK(\cdot, x_i)$$
for a given loss function $\ell$.
The symmetric SignGD and SGD kernels train $\alpha_i$ via gradient descent to minimize a regularized logistic loss on the downstream task.
We search over a grid of regularization strengths chosen proportional to $\|\gK\|_{\text{op}}$, see Table \ref{tab:hyperparameters}.
For a test input $x$, the kernel outputs the prediction $h(x) = \sum_i \alpha_i \gK(x,x_i)$.

\paragraph{Asymmetric Kernel}
We write how to solve the kernel regression problem with an asymmetric kernel, developed in \cite{he2022learning}, here. Consider the augmented linear system: 
$$
\begin{bmatrix} I/\gamma & H \\ H^\top & I/\gamma\end{bmatrix}\begin{bmatrix}  \alpha \\ \beta\end{bmatrix} = \begin{bmatrix} 1 \\ 1 \end{bmatrix}
$$
where $H_{ij} = y_i\phi_s(x_i)^\top \phi_t(x_j) y_j$ with $\phi_s$ and $\phi_t$ as the two different feature maps and $y_i$ as the label for the $i$th example. In our setting, $\phi_s$ is the gradient of the datapoint, and $\phi_t$ is the sign of the gradient.
Define $\omega^*$ and $\nu^*$ as 
\begin{align*}
	\omega^* = \sum_i \beta_i^* y_i\phi_t(x_i) \\
	\nu^* = \sum_i \alpha_i^* y_i\phi_s(x_i)	
\end{align*}
Solving this system yields two discriminant functions:
\begin{align*}
	f_s(x) &= K(x,X)(\beta^*\odot Y) \\ %
	f_t(x) &= K(X,x)(\alpha^*\odot Y) %
\end{align*}

where $K(x_i,x_j) = \langle \phi_s(x_i), \phi_t(x_j)\rangle$.

We can thus create one discriminant function as $cf_s(x) + (1-c)f_t(x)$ where $c\in[0, 1]$ is some hyperparameter. When $\phi_s=\phi_t$, we see that $f_s=f_t$ and we reduce to the standard kernel problem (though with repeated equations). Note that per \citet{he2022learning}, this system is only meaningful in terms of stationary points when training $\alpha$ and $\beta$ using the least squares loss.

We now leverage some specific knowledge about the NTK setting. In particular, we know that we should only use $f_s$ as the predictor in order to correctly represent a new test input in the kernel analog for SignGD.

\paragraph{Hyperparameters and Implementation}

\begin{table*}[h]
\centering
\resizebox{0.65\columnwidth}{!}{%
\begin{tabular}{lrc}
\toprule
Experiment & Hyperparameters & Values \\
\midrule
SGD FT & Batch size & $\{ 2,4,8 \}$ $\times$ \\
& Learning rate & $\{1\mathrm{e}{-4}, 5\mathrm{e}{-4}, 1\mathrm{e}{-3}, 5\mathrm{e}{-3}, 1\mathrm{e}{-2} \}$ \\
\cmidrule{2-3}
SGD-LoRA FT & Batch size & $\{ 4, 16 \}$ $\times$ \\
& Learning rate & $\{1\mathrm{e}{-4}, 1\mathrm{e}{-3}, 1\mathrm{e}{-2}\}$ $\times$ \\
& ($r_{LoRA}$, $\alpha_{LoRA}$)& $\{(8, 16)\}$ \\
\midrule
Adam FT & Batch size & $\{ 2,4,8 \}$ $\times$ \\
& Learning rate & $\{1\mathrm{e}{-5}, 2\mathrm{e}{-5}, 5\mathrm{e}{-5}\}$ \\
\cmidrule{2-3}
Adam-LoRA FT & Batch size & $\{ 4,16 \}$ $\times$ \\
& Learning rate & $\{1\mathrm{e}{-5}, 4\mathrm{e}{-5}, 4\mathrm{e}{-4}\}$ $\times$ \\
& ($r_{LoRA}$, $\alpha_{LoRA}$)& $\{(8, 16)\}$ \\
\midrule
\midrule
$\ksgd$, $\ksigngd$ & Kernel regularization & $\{0, 0.001, 0.01, 0.1, 1\}$ $\times$ \\
& $f_0$ scaling & $\{10, 100, 1000, 10000, \infty\}$ \\
\cmidrule{2-3}
$\kasigngd$ &  Kernel regularization & $\{0, 0.001, 0.01, 0.1, 1\}$ $\times$ \\
& $f_0$ scaling & $\{10, 100, 1000, 10000, \infty\}$ $\times$ \\
& Kernel $\gamma$ & $\{0.01, 0.1, 1, 10\}$ $\times$ \\
& Kernel $c$ & $\{1\}$ \\
\bottomrule
\end{tabular}
}
\caption{The hyperparameter grids used in our experiments.}
\label{tab:hyperparameters}
\end{table*}

We follow \cite{gao2020making} in using the few-shot validation set to search over hyperparameters and finding the best hyperparameter per few-shot dataset.
We use value ranges given by \cite{gao2020making} and \cite{hu2021lora}, and search over a wider range of values for SGD. Table \ref{tab:hyperparameters} shows the hyperparameter grids for fine-tuning and the kernel method. We fine-tune without weight decay and a learning rate schedule with a linear decay and no warmup.

\citet{gao2020making} train for $1000$ steps in the 16-shot setting, and validate the performance every $100$ steps to take the best checkpoints. As we consider varying values of $k$, we use the formula of training for $32 k C$ steps and validating every $4 k C$ steps, where $C$ is the number of classes in the dataset. This gives a comparable number of training and validation steps for binary tasks in the 16-shot setting.

\section{Additional Experimental Results}
Tables \ref{tab:noprompt} and \ref{tab:more_linearization_measurement} contain the numerical results corresponding to Figures \ref{fig:prompt_vs_noprompt} and \ref{fig:linearization_figure} respectively, and also report results for $k=64$.
Table \ref{tab:linearization_kernel_distance} measures how well the Fixed Features property holds for different tasks. 
A smaller value suggests that the condition for kernel behavior (\Cref{def:kernel_regime}) is satisfied more strongly.

\begin{figure*}[t]
    \centering
    \begin{subfigure}[b]{\textwidth}
        \centering
        \caption{SGD-FT}
        \includegraphics[width=\textwidth]{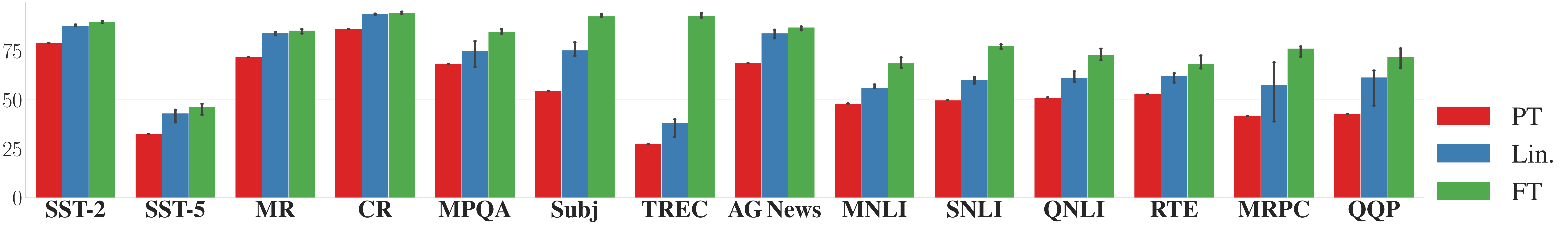}
    \end{subfigure}
    \begin{subfigure}[b]{\textwidth}
        \centering
        \caption{Adam-FT}
        \includegraphics[width=\textwidth]{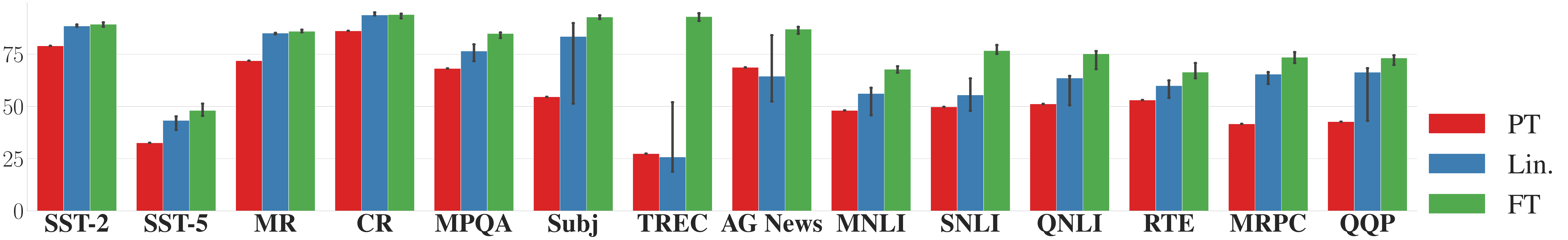}
        \label{fig:linearization_figure_adam}
    \end{subfigure}
    \caption{Accuracies of zero-shot pre-trained model (PT), linearized model (Lin., see \Cref{def:kernel_regime}) and fine-tuned model (FT). Tasks that induce the Linearization property of kernel behavior (\Cref{def:kernel_regime}) will show that Lin. performance recovers a substantial amount of the performance of SGD-FT and Adam-FT respectively. 
    We plot the median and range of the test accuracies across 5 seeds and data splits for $k=64$.
    }
\label{fig:linearization_figure}
\end{figure*}
\begin{table}[!h]
    \centering
    \begin{subtable}[h]{1.0\textwidth}
        \centering
        \resizebox{0.82\textwidth}{!}{
            \begin{tabular}{r|cc|cc|cc|cc}
            \toprule
              & \multicolumn{2}{c|}{\textbf{SST-2} } & \multicolumn{2}{c|}{\textbf{SST-5}} & \multicolumn{2}{c|}{\textbf{MR}} & \multicolumn{2}{c}{\textbf{CR} }  \\ 
                  $k$-shot        &  Lin. & FT & Lin. & FT & Lin. & FT &  Lin. & FT  \\
             \midrule
             0 & \multicolumn{2}{c|}{------ \textit{79.0} ------} & \multicolumn{2}{c|}{------ \textit{32.6} ------} & \multicolumn{2}{c|}{------ \textit{71.9} ------} & \multicolumn{2}{c}{------ \textit{86.2} ------} \\
            16                  & 87.5$_{(1.3)}$ & 88.3$_{(1.2)}$  & 41.8$_{(4.1)}$ & 45.4$_{(2.6)}$ & 84.3$_{(1.8)}$ & 81.3$_{(6.1)}$ & 93.3$_{(0.6)}$ & 93.0$_{(1.6)}$       \\
            64                  & 88.6$_{(0.4)}$ &  89.3$_{(0.7)}$ & 42.9$_{(2.2)}$ & 48.5$_{(2.0)}$ & 85.0$_{(0.2)}$ & 86.0$_{(0.4)}$     & 94.0$_{(0.5)}$ & 93.7$_{(0.8)}$    \\
            \midrule
            \midrule
            & \multicolumn{2}{c|}{\textbf{MQPA}} & \multicolumn{2}{c|}{\textbf{Subj}} & \multicolumn{2}{c|}{\textbf{TREC}} & \multicolumn{2}{c}{\textbf{AG News}} \\
            $k$-shot  & Lin. & FT & Lin. & FT & Lin. & FT & Lin. & FT \\
            \midrule
            0 & \multicolumn{2}{c|}{------ \textit{68.2} ------} & \multicolumn{2}{c|}{------ \textit{54.6} ------} & \multicolumn{2}{c|}{------ \textit{27.4} ------} & \multicolumn{2}{c}{------ \textit{68.7} ------} \\
            16 & 75.6$_{(3.1)}$ & 82.8$_{(2.2)}$     & 82.9$_{(4.7)}$  & 87.4$_{(2.1)}$              &  30.4$_{(7.2)}$ & 79.6$_{(6.1)}$  & 57.8$_{(18.3)}$ & 84.0$_{(1.6)}$    \\
            64 & 75.6$_{(2.3)}$ & 85.0$_{(0.2)}$     & 78.9$_{(14.0)}$ & 92.7$_{(0.6)}$ &  31.2$_{(13.0)}$ & 92.6$_{(1.3)}$  & 67.5$_{(12.2)}$ & 86.8$_{(1.1)}$ \\
            \bottomrule
        \end{tabular}
        }
        \caption{Single-sentence tasks. \vspace{1.0em}}
    \end{subtable}
        
    \begin{subtable}[h]{1.0\textwidth}
        \centering
        \resizebox{0.64\textwidth}{!}{
        \begin{tabular}{r|cc|cc|cc}
            \toprule
            & \multicolumn{2}{c|}{\textbf{MNLI}} & \multicolumn{2}{c|}{\textbf{SNLI}} & \multicolumn{2}{c}{\textbf{QNLI}} \\
            $k$-shot  & Lin. & FT & Lin. & FT & Lin. & FT  \\
            \midrule
            0 & \multicolumn{2}{c|}{------ \textit{48.1} ------} & \multicolumn{2}{c|}{------ \textit{49.8} ------} & \multicolumn{2}{c}{------ \textit{51.2} ------} \\
            16 & 43.6$_{(6.4)}$ & 56.8$_{(2.9)}$ &  47.2$_{(9.3)}$ & 64.6$_{(4.1)}$ & 57.5$_{(2.3)}$ & 63.1$_{(3.5)}$  \\
            64    & 55.1$_{(4.8)}$ & 67.9$_{(1.0)}$ &   56.9$_{(5.7)}$ & 76.9$_{(1.4)}$ & 60.4$_{(5.3)}$ & 74.2$_{(3.2)}$  \\
            \midrule
            \midrule
            & \multicolumn{2}{c|}{\textbf{RTE}} & \multicolumn{2}{c|}{\textbf{MRPC}} & \multicolumn{2}{c}{\textbf{QQP}} \\ 
            $k$-shot  & Lin. & FT & Lin. & FT & Lin. & FT  \\
            \midrule
            0 & \multicolumn{2}{c|}{------ \textit{53.1} ------} & \multicolumn{2}{c|}{------ \textit{41.7} ------} & \multicolumn{2}{c}{------ \textit{42.7} ------} \\
            16  &  55.4$_{(6.7)}$  & 57.6$_{(6.3)}$      & 57.7$_{(11.6)}$ & 68.9$_{(2.4)}$ & 57.5$_{(10.3)}$ & 61.7$_{(6.5)}$  \\
            64     &  59.6$_{(2.9)}$  & 67.3$_{(2.7)}$   &  64.2$_{(2.2)}$  & 73.8$_{(1.7)}$ & 61.7$_{(9.4)}$  & 72.7$_{(1.8)}$  \\
            \bottomrule
            \end{tabular}
        }
        \caption{Sentence-pair tasks.}
    \end{subtable}
    \caption{Accuracies of pre-trained model (0-shot), linearized model (Lin., see \Cref{def:kernel_regime}) and fine-tuned model (FT). Tasks that exhibit the Linearization property of kernel behavior (\Cref{def:kernel_regime}) during fine-tuning will show that Lin. performance recovers a substantial amount of the gain in performance achieved by performing fine-tuning with Adam. Accuracies are averaged across 5 fine-tuning seeds for each value of $k$ and measured on the test set. This table corresponds to the bar chart in \Cref{fig:linearization_figure_adam}.}
    \label{tab:more_linearization_measurement}
\end{table}

\begin{table}[!h]
\centering

\begin{subtable}[h]{1.0\textwidth}
    \centering
\resizebox{0.84\textwidth}{!}{
\begin{tabular}{rllcccccccccccccc}
\toprule
$k$-shot & Prompt & Method & \textbf{SST-2} & \textbf{SST-5} & \textbf{MR} & \textbf{CR} & \textbf{MPQA} & \textbf{Subj} & \textbf{TREC} & \textbf{AG News} \\
\midrule
 \multirow[t]{10}{*}{16} & \multirow[t]{5}{*}{Prompt} & SGD FT & 89.0$_{(1.5)}$ & 44.6$_{(1.4)}$ & 83.4$_{(2.5)}$ & 93.3$_{(0.2)}$ & 83.3$_{(1.3)}$ & 88.5$_{(2.6)}$ & 80.3$_{(7.2)}$ & 84.2$_{(1.1)}$ \\
 &  & {\cellcolor{gblue!10}$\mathcal{K}^{\text{(SGD)}}$} & {\cellcolor{gblue!10}88.3$_{(0.3)}$} & {\cellcolor{gblue!10}43.6$_{(2.2)}$} & {\cellcolor{gblue!10}84.7$_{(1.5)}$} & {\cellcolor{gblue!10}93.2$_{(0.9)}$} & {\cellcolor{gblue!10}76.4$_{(2.7)}$} & {\cellcolor{gblue!10}88.6$_{(1.3)}$} & {\cellcolor{gblue!10}56.0$_{(9.2)}$} & {\cellcolor{gblue!10}82.1$_{(2.0)}$} \\
\cmidrule{3-11} 
 &  & Adam FT & 88.3$_{(1.2)}$ & 45.4$_{(2.6)}$ & 81.3$_{(6.1)}$ & 93.0$_{(1.6)}$ & 82.8$_{(2.2)}$ & 87.4$_{(2.1)}$ & 79.6$_{(6.1)}$ & 84.0$_{(1.6)}$ \\
 &  & {\cellcolor{ggreen!7}$\mathcal{K}^{\text{(SignGD)}}$} & {\cellcolor{ggreen!7}88.3$_{(0.5)}$} & {\cellcolor{ggreen!7}42.2$_{(3.9)}$} & {\cellcolor{ggreen!7}84.3$_{(1.5)}$} & {\cellcolor{ggreen!7}93.7$_{(0.5)}$} & {\cellcolor{ggreen!7}76.7$_{(3.3)}$} & {\cellcolor{ggreen!7}89.2$_{(2.0)}$} & {\cellcolor{ggreen!7}58.1$_{(6.5)}$} & {\cellcolor{ggreen!7}82.3$_{(1.6)}$} \\
 &  & {\cellcolor{ggreen!7}$\mathcal{K}^{\text{(A-)SignGD)}}$} & {\cellcolor{ggreen!7}88.3$_{(0.4)}$} & {\cellcolor{ggreen!7}43.7$_{(1.7)}$} & {\cellcolor{ggreen!7}84.9$_{(1.1)}$} & {\cellcolor{ggreen!7}93.4$_{(0.5)}$} & {\cellcolor{ggreen!7}74.6$_{(3.5)}$} & {\cellcolor{ggreen!7}88.6$_{(1.8)}$} & {\cellcolor{ggreen!7}20.7$_{(4.2)}$} & {\cellcolor{ggreen!7}83.6$_{(1.0)}$} \\
\cmidrule{2-11} 
 & \multirow[t]{5}{*}{Standard} & SGD FT & 79.7$_{(4.5)}$ & 36.1$_{(3.7)}$ & 64.8$_{(5.2)}$ & 86.6$_{(2.6)}$ & 69.1$_{(6.8)}$ & 89.2$_{(0.7)}$ & 62.7$_{(3.8)}$ & 82.3$_{(0.4)}$ \\
 &  & {\cellcolor{gblue!10}$\mathcal{K}^{\text{(SGD)}}$} & {\cellcolor{gblue!10}62.3$_{(6.4)}$} & {\cellcolor{gblue!10}32.0$_{(1.5)}$} & {\cellcolor{gblue!10}61.2$_{(4.0)}$} & {\cellcolor{gblue!10}67.5$_{(2.3)}$} & {\cellcolor{gblue!10}62.7$_{(2.3)}$} & {\cellcolor{gblue!10}86.7$_{(1.5)}$} & {\cellcolor{gblue!10}58.7$_{(6.0)}$} & {\cellcolor{gblue!10}81.3$_{(1.5)}$} \\
\cmidrule{3-11} 
 &  & Adam FT & 79.3$_{(1.9)}$ & 37.9$_{(5.2)}$ & 69.0$_{(6.0)}$ & 83.9$_{(5.2)}$ & 69.5$_{(6.8)}$ & 89.5$_{(1.0)}$ & 74.4$_{(2.4)}$ & 82.7$_{(2.1)}$ \\
 &  & {\cellcolor{ggreen!7}$\mathcal{K}^{\text{(SignGD)}}$} & {\cellcolor{ggreen!7}61.3$_{(8.6)}$} & {\cellcolor{ggreen!7}32.2$_{(2.2)}$} & {\cellcolor{ggreen!7}61.4$_{(4.0)}$} & {\cellcolor{ggreen!7}72.6$_{(3.1)}$} & {\cellcolor{ggreen!7}60.9$_{(3.6)}$} & {\cellcolor{ggreen!7}87.8$_{(1.7)}$} & {\cellcolor{ggreen!7}63.5$_{(3.8)}$} & {\cellcolor{ggreen!7}81.6$_{(1.2)}$} \\
 &  & {\cellcolor{ggreen!7}$\mathcal{K}^{\text{(A-)SignGD)}}$} & {\cellcolor{ggreen!7}59.1$_{(11.4)}$} & {\cellcolor{ggreen!7}31.9$_{(2.0)}$} & {\cellcolor{ggreen!7}58.3$_{(8.8)}$} & {\cellcolor{ggreen!7}72.4$_{(4.1)}$} & {\cellcolor{ggreen!7}60.7$_{(4.6)}$} & {\cellcolor{ggreen!7}87.7$_{(1.7)}$} & {\cellcolor{ggreen!7}64.6$_{(4.1)}$} & {\cellcolor{ggreen!7}81.1$_{(1.5)}$} \\
\cmidrule{1-11} 
 \multirow[t]{10}{*}{64} & \multirow[t]{5}{*}{Prompt} & SGD FT & 89.7$_{(0.4)}$ & 45.8$_{(2.1)}$ & 85.8$_{(1.0)}$ & 94.3$_{(0.5)}$ & 84.8$_{(0.8)}$ & 92.9$_{(0.5)}$ & 93.2$_{(1.0)}$ & 86.8$_{(0.7)}$ \\
 &  & {\cellcolor{gblue!10}$\mathcal{K}^{\text{(SGD)}}$} & {\cellcolor{gblue!10}89.2$_{(1.0)}$} & {\cellcolor{gblue!10}46.0$_{(1.3)}$} & {\cellcolor{gblue!10}86.4$_{(0.6)}$} & {\cellcolor{gblue!10}93.7$_{(0.4)}$} & {\cellcolor{gblue!10}81.2$_{(0.9)}$} & {\cellcolor{gblue!10}91.4$_{(0.7)}$} & {\cellcolor{gblue!10}77.8$_{(2.3)}$} & {\cellcolor{gblue!10}85.6$_{(0.7)}$} \\
\cmidrule{3-11} 
 &  & Adam FT & 89.3$_{(0.7)}$ & 48.5$_{(2.0)}$ & 86.0$_{(0.4)}$ & 93.7$_{(0.8)}$ & 84.6$_{(0.9)}$ & 92.7$_{(0.6)}$ & 92.6$_{(1.3)}$ & 86.8$_{(1.1)}$ \\
 &  & {\cellcolor{ggreen!7}$\mathcal{K}^{\text{(SignGD)}}$} & {\cellcolor{ggreen!7}89.1$_{(0.5)}$} & {\cellcolor{ggreen!7}49.1$_{(1.6)}$} & {\cellcolor{ggreen!7}85.6$_{(1.0)}$} & {\cellcolor{ggreen!7}93.9$_{(0.2)}$} & {\cellcolor{ggreen!7}79.0$_{(5.8)}$} & {\cellcolor{ggreen!7}92.4$_{(0.5)}$} & {\cellcolor{ggreen!7}82.0$_{(1.4)}$} & {\cellcolor{ggreen!7}85.9$_{(0.7)}$} \\
 &  & {\cellcolor{ggreen!7}$\mathcal{K}^{\text{(A-)SignGD)}}$} & {\cellcolor{ggreen!7}88.9$_{(0.9)}$} & {\cellcolor{ggreen!7}43.6$_{(2.2)}$} & {\cellcolor{ggreen!7}85.6$_{(1.0)}$} & {\cellcolor{ggreen!7}94.0$_{(0.3)}$} & {\cellcolor{ggreen!7}81.8$_{(1.1)}$} & {\cellcolor{ggreen!7}91.8$_{(1.1)}$} & {\cellcolor{ggreen!7}22.8$_{(2.9)}$} & {\cellcolor{ggreen!7}86.2$_{(0.3)}$} \\
\cmidrule{2-11} 
 & \multirow[t]{5}{*}{Standard} & SGD FT & 85.6$_{(3.6)}$ & 41.1$_{(2.1)}$ & 83.4$_{(1.7)}$ & 92.7$_{(1.2)}$ & 83.5$_{(2.1)}$ & 92.6$_{(0.4)}$ & 86.8$_{(1.8)}$ & 86.8$_{(0.8)}$ \\
 &  & {\cellcolor{gblue!10}$\mathcal{K}^{\text{(SGD)}}$} & {\cellcolor{gblue!10}77.7$_{(2.8)}$} & {\cellcolor{gblue!10}35.8$_{(0.7)}$} & {\cellcolor{gblue!10}73.6$_{(2.0)}$} & {\cellcolor{gblue!10}82.6$_{(4.4)}$} & {\cellcolor{gblue!10}74.9$_{(2.2)}$} & {\cellcolor{gblue!10}90.1$_{(1.0)}$} & {\cellcolor{gblue!10}81.9$_{(2.0)}$} & {\cellcolor{gblue!10}85.6$_{(0.6)}$} \\
\cmidrule{3-11} 
 &  & Adam FT & 86.2$_{(2.3)}$ & 41.0$_{(1.7)}$ & 83.9$_{(1.9)}$ & 92.6$_{(1.0)}$ & 83.5$_{(1.8)}$ & 92.9$_{(0.5)}$ & 91.5$_{(1.4)}$ & 87.5$_{(0.6)}$ \\
 &  & {\cellcolor{ggreen!7}$\mathcal{K}^{\text{(SignGD)}}$} & {\cellcolor{ggreen!7}79.6$_{(1.7)}$} & {\cellcolor{ggreen!7}35.3$_{(3.1)}$} &{\cellcolor{ggreen!7}75.8$_{(2.0)}$} & {\cellcolor{ggreen!7}83.0$_{(4.7)}$} & {\cellcolor{ggreen!7}75.0$_{(2.1)}$} & {\cellcolor{ggreen!7}90.9$_{(1.0)}$} & {\cellcolor{ggreen!7}82.5$_{(1.8)}$} & {\cellcolor{ggreen!7}85.9$_{(1.0)}$} \\
 &  & {\cellcolor{ggreen!7}$\mathcal{K}^{\text{(A-)SignGD)}}$} & {\cellcolor{ggreen!7}78.7$_{(2.3)}$} & {\cellcolor{ggreen!7}36.8$_{(2.3)}$} & {\cellcolor{ggreen!7}76.5$_{(3.2)}$} & {\cellcolor{ggreen!7}85.6$_{(3.8)}$} & {\cellcolor{ggreen!7}75.2$_{(1.9)}$} & {\cellcolor{ggreen!7}91.1$_{(1.1)}$} & {\cellcolor{ggreen!7}84.6$_{(1.5)}$} & {\cellcolor{ggreen!7}86.2$_{(0.8)}$} \\
 
\cmidrule{1-11} 
 \multirow[t]{4}{*}{512} & \multirow[t]{2}{*}{Prompt} & SGD FT & 92.0$_{(0.9)}$ & 53.5$_{(1.5)}$ & 88.8$_{(0.0)}$ & 94.3$_{(0.4)}$ & 88.5$_{(0.1)}$ & 95.4$_{(0.1)}$ & 97.2$_{(0.4)}$ & 89.9$_{(0.7)}$ \\
 &  & {\cellcolor{gblue!10}$\mathcal{K}^{\text{(SGD)}}$} & {\cellcolor{gblue!10}91.0$_{(0.2)}$} & {\cellcolor{gblue!10}49.8$_{(0.4)}$} & {\cellcolor{gblue!10}88.0$_{(0.9)}$} & {\cellcolor{gblue!10}94.4$_{(0.2)}$} & {\cellcolor{gblue!10}84.4$_{(0.9)}$} & {\cellcolor{gblue!10}93.5$_{(0.1)}$} & {\cellcolor{gblue!10}88.2$_{(0.8)}$} & {\cellcolor{gblue!10}88.4$_{(0.5)}$}\\
\cmidrule{2-11} 
 & \multirow[t]{2}{*}{Standard} & SGD FT & 91.4$_{(0.2)}$ & 50.2$_{(1.6)}$ & 88.8$_{(0.4)}$ & 95.4$_{(0.3)}$ & 88.1$_{(0.5)}$ & 95.0$_{(0.7)}$ & 97.2$_{(0.6)}$ & 90.1$_{(0.4)}$ \\
 &  & {\cellcolor{gblue!10}$\mathcal{K}^{\text{(SGD)}}$} & {\cellcolor{gblue!10}85.9$_{(1.6)}$} & {\cellcolor{gblue!10}45.4$_{(1.0)}$} & {\cellcolor{gblue!10}83.1$_{(1.1)}$} & {\cellcolor{gblue!10}92.2$_{(0.9)}$} & {\cellcolor{gblue!10}83.4$_{(0.5)}$} & {\cellcolor{gblue!10}92.3$_{(0.1)}$} & {\cellcolor{gblue!10}93.3$_{(1.5)}$} & {\cellcolor{gblue!10}89.1$_{(0.2)}$}\\
\bottomrule
\end{tabular}
}
\caption{Single-sentence tasks \vspace{1.0em}}
\end{subtable}

\begin{subtable}[h]{1.0\textwidth}
    \centering
\resizebox{0.69\textwidth}{!}{
\begin{tabular}{rllcccccccccccccc}
\toprule
$k$-shot & Prompt & Method & \textbf{MNLI} & \textbf{SNLI} & \textbf{QNLI} & \textbf{RTE} & \textbf{MRPC} & \textbf{QQP} \\
\midrule
 \multirow[t]{10}{*}{16} & \multirow[t]{5}{*}{Prompt} & SGD FT & 59.2$_{(2.7)}$ & 65.7$_{(2.7)}$ & 62.1$_{(3.1)}$ & 60.0$_{(5.5)}$ & 73.9$_{(2.7)}$ & 62.1$_{(2.3)}$ \\
 &  & {\cellcolor{gblue!10}$\mathcal{K}^{\text{(SGD)}}$} & {\cellcolor{gblue!10}53.0$_{(3.0)}$} & {\cellcolor{gblue!10}57.8$_{(2.3)}$} & {\cellcolor{gblue!10}60.1$_{(3.3)}$} & {\cellcolor{gblue!10}60.0$_{(4.7)}$} & {\cellcolor{gblue!10}73.4$_{(5.6)}$} & {\cellcolor{gblue!10}58.2$_{(0.9)}$} \\
\cmidrule{3-9}
 &  & Adam FT & 56.8$_{(2.9)}$ & 64.6$_{(4.1)}$ & 63.1$_{(3.5)}$ & 57.6$_{(6.3)}$ & 77.6$_{(3.1)}$ & 61.8$_{(4.5)}$ \\
 &  & {\cellcolor{ggreen!7}$\mathcal{K}^{\text{(SignGD)}}$} & {\cellcolor{ggreen!7}53.8$_{(1.2)}$} & {\cellcolor{ggreen!7}54.9$_{(2.7)}$} & {\cellcolor{ggreen!7}59.5$_{(3.1)}$} & {\cellcolor{ggreen!7}55.4$_{(4.2)}$} & {\cellcolor{ggreen!7}75.6$_{(1.2)}$} & {\cellcolor{ggreen!7}60.7$_{(2.2)}$} \\
 &  & {\cellcolor{ggreen!7}$\mathcal{K}^{\text{(A-)SignGD)}}$} & {\cellcolor{ggreen!7}51.9$_{(4.0)}$} & {\cellcolor{ggreen!7}54.9$_{(3.1)}$} & {\cellcolor{ggreen!7}56.0$_{(1.9)}$} & {\cellcolor{ggreen!7}59.8$_{(4.0)}$} & {\cellcolor{ggreen!7}75.2$_{(2.6)}$} & {\cellcolor{ggreen!7}59.4$_{(2.0)}$} \\
\cmidrule{2-9}
 & \multirow[t]{5}{*}{Standard} & SGD FT & 35.2$_{(1.3)}$ & 41.3$_{(2.2)}$ & 52.5$_{(5.4)}$ & 50.2$_{(2.1)}$ & 73.7$_{(6.3)}$ & 55.3$_{(5.2)}$ \\
 &  & {\cellcolor{gblue!10}$\mathcal{K}^{\text{(SGD)}}$} & {\cellcolor{gblue!10}34.9$_{(1.8)}$} & {\cellcolor{gblue!10}39.6$_{(3.3)}$} & {\cellcolor{gblue!10}50.3$_{(1.4)}$} & {\cellcolor{gblue!10}48.7$_{(2.0)}$} & {\cellcolor{gblue!10}69.2$_{(6.9)}$} & {\cellcolor{gblue!10}50.8$_{(5.0)}$} \\
\cmidrule{3-9}
 &  & Adam FT & 38.7$_{(3.5)}$ & 42.9$_{(3.2)}$ & 57.6$_{(4.2)}$ & 51.1$_{(3.8)}$ & 75.6$_{(7.1)}$ & 58.2$_{(6.5)}$ \\
 &  & {\cellcolor{ggreen!7}$\mathcal{K}^{\text{(SignGD)}}$} & {\cellcolor{ggreen!7}36.1$_{(1.3)}$} & {\cellcolor{ggreen!7}41.7$_{(2.4)}$} & {\cellcolor{ggreen!7}51.9$_{(1.5)}$} &  
 {\cellcolor{ggreen!7}48.2$_{(3.4)}$} & {\cellcolor{ggreen!7}73.3$_{(5.3)}$} & {\cellcolor{ggreen!7}52.4$_{(5.1)}$} \\
 &  & {\cellcolor{ggreen!7}$\mathcal{K}^{\text{(A-)SignGD)}}$} & {\cellcolor{ggreen!7}34.9$_{(1.4)}$} & {\cellcolor{ggreen!7}41.7$_{(2.5)}$} & {\cellcolor{ggreen!7}52.6$_{(2.5)}$} & {\cellcolor{ggreen!7}48.2$_{(2.5)}$} & {\cellcolor{ggreen!7}73.8$_{(6.2)}$} & {\cellcolor{ggreen!7}50.8$_{(8.8)}$} \\
\cmidrule{1-9}
 \multirow[t]{10}{*}{64} & \multirow[t]{5}{*}{Prompt} & SGD FT & 68.7$_{(1.7)}$ & 77.3$_{(0.9)}$ & 72.8$_{(2.2)}$ & 68.9$_{(2.5)}$ & 82.8$_{(1.2)}$ & 69.2$_{(1.3)}$ \\
 &  & {\cellcolor{gblue!10}$\mathcal{K}^{\text{(SGD)}}$} & {\cellcolor{gblue!10}60.4$_{(1.8)}$} & {\cellcolor{gblue!10}65.5$_{(1.6)}$} & {\cellcolor{gblue!10}67.3$_{(1.6)}$} & {\cellcolor{gblue!10}66.5$_{(2.5)}$} & {\cellcolor{gblue!10}79.2$_{(2.5)}$} & {\cellcolor{gblue!10}66.4$_{(1.7)}$} \\
\cmidrule{3-9}
 &  & Adam FT & 67.9$_{(1.0)}$ & 76.9$_{(1.4)}$ & 74.2$_{(3.2)}$ & 67.3$_{(2.7)}$ & 80.9$_{(1.2)}$ & 69.8$_{(0.6)}$ \\
 &  & {\cellcolor{ggreen!7}$\mathcal{K}^{\text{(SignGD)}}$} & {\cellcolor{ggreen!7}60.8$_{(1.7)}$} & {\cellcolor{ggreen!7}64.1$_{(2.3)}$} & {\cellcolor{ggreen!7}65.4$_{(1.7)}$} & {\cellcolor{ggreen!7}63.8$_{(1.8)}$} & {\cellcolor{ggreen!7}77.4$_{(2.3)}$} & {\cellcolor{ggreen!7}63.7$_{(4.4)}$} \\
 &  & {\cellcolor{ggreen!7}$\mathcal{K}^{\text{(A-)SignGD)}}$} & {\cellcolor{ggreen!7}58.5$_{(1.7)}$} & {\cellcolor{ggreen!7}66.8$_{(1.1)}$} & {\cellcolor{ggreen!7}66.5$_{(1.1)}$} & {\cellcolor{ggreen!7}63.8$_{(2.2)}$} & {\cellcolor{ggreen!7}77.3$_{(2.0)}$} & {\cellcolor{ggreen!7}66.1$_{(3.4)}$} \\
\cmidrule{2-9}
 & \multirow[t]{5}{*}{Standard} & SGD FT & 50.0$_{(5.0)}$ & 61.9$_{(4.5)}$ & 65.4$_{(4.2)}$ & 53.6$_{(2.5)}$ & 78.7$_{(1.1)}$ & 64.8$_{(3.5)}$ \\
 &  & {\cellcolor{gblue!10}$\mathcal{K}^{\text{(SGD)}}$} & {\cellcolor{gblue!10}42.6$_{(1.7)}$} & {\cellcolor{gblue!10}50.1$_{(1.7)}$} & {\cellcolor{gblue!10}54.4$_{(1.5)}$} & {\cellcolor{gblue!10}50.0$_{(4.4)}$} & {\cellcolor{gblue!10}72.2$_{(5.8)}$} & {\cellcolor{gblue!10}48.4$_{(19.3)}$} \\
\cmidrule{3-9}
 &  & Adam FT & 58.0$_{(2.6)}$ & 67.8$_{(2.0)}$ & 67.9$_{(7.2)}$ & 53.9$_{(4.2)}$ & 80.1$_{(1.4)}$ & 66.8$_{(3.1)}$ \\
 &  & {\cellcolor{ggreen!7}$\mathcal{K}^{\text{(SignGD)}}$} & {\cellcolor{ggreen!7}41.7$_{(2.1)}$} & {\cellcolor{ggreen!7}50.5$_{(2.1)}$} & {\cellcolor{ggreen!7}56.6$_{(1.9)}$} &  
 {\cellcolor{ggreen!7}52.7$_{(3.8)}$} & {\cellcolor{ggreen!7}77.6$_{(4.2)}$} & {\cellcolor{ggreen!7}61.3$_{(2.0)}$} \\
 &  & {\cellcolor{ggreen!7}$\mathcal{K}^{\text{(A-)SignGD)}}$} & {\cellcolor{ggreen!7}42.8$_{(1.7)}$} & {\cellcolor{ggreen!7}49.1$_{(2.9)}$} & {\cellcolor{ggreen!7}55.3$_{(3.7)}$} & {\cellcolor{ggreen!7}52.9$_{(4.5)}$} & {\cellcolor{ggreen!7}74.5$_{(2.5)}$} & {\cellcolor{ggreen!7}62.3$_{(1.9)}$} \\
\cmidrule{1-9}
 \multirow[t]{4}{*}{512} & \multirow[t]{2}{*}{Prompt} & SGD FT & 78.4$_{(0.3)}$ & 83.9$_{(0.3)}$ & 81.9$_{(1.2)}$ & 76.3$_{(0.6)}$ & 89.2$_{(0.1)}$ & 75.2$_{(1.1)}$ \\
 &  & {\cellcolor{gblue!10}$\mathcal{K}^{\text{(SGD)}}$} & {\cellcolor{gblue!10}67.4$_{(0.2)}$} & {\cellcolor{gblue!10}74.6$_{(0.3)}$} & {\cellcolor{gblue!10}76.1$_{(0.9)}$} & {\cellcolor{gblue!10}74.2$_{(1.2)}$} & {\cellcolor{gblue!10}80.7$_{(1.7)}$} & {\cellcolor{gblue!10}72.0$_{(0.9)}$} \\
\cmidrule{2-9}
 & \multirow[t]{2}{*}{Standard} & SGD FT & 77.8$_{(1.1)}$ & 82.9$_{(0.6)}$ & 81.0$_{(0.5)}$ & 70.9$_{(1.7)}$ & 90.2$_{(0.7)}$ & 75.7$_{(0.9)}$ \\
 &  & {\cellcolor{gblue!10}$\mathcal{K}^{\text{(SGD)}}$} & {\cellcolor{gblue!10}57.6$_{(3.6)}$} & {\cellcolor{gblue!10}67.0$_{(1.2)}$} & {\cellcolor{gblue!10}68.4$_{(0.4)}$} & {\cellcolor{gblue!10}55.7$_{(1.7)}$} & {\cellcolor{gblue!10}78.7$_{(2.2)}$} & {\cellcolor{gblue!10}69.1$_{(1.3)}$} \\
\bottomrule
\end{tabular}
}
\caption{Sentence-pair tasks \vspace{-0.5em}} 
\end{subtable}

\caption{Fine-tuning performance in the standard FT setting, where the contextual embedding of the \texttt{[CLS]} token is used for classification, and the prompt-based FT setting, where a prompt is added and the embedding for the \texttt{[MASK]} token is used (see \Cref{sec:prelims}). 
In standard FT, we initialize the new classification head (i.e., $\Gamma$) using the linear probing solution. 
This table gives the figures in ~\Cref{fig:prompt_vs_noprompt}, and also relates SGD fine-tuning performance to the more common fine-tuning with Adam. We report F1 for MRPC and QQP and accuracy otherwise, and average the metrics over 5 seeds for 16-shot and 64-shot, and 3 seeds for 512-shot.
}

\label{tab:noprompt}

\end{table}

\begin{table}[h]
    \centering
    
        \begin{subtable}[h]{1.0\textwidth}
        \centering
    \resizebox{1.0\textwidth}{!}{
        \begin{tabular}{lr|cccccccc}
        \toprule
        Method & $k$-shot  & \textbf{SST-2} & \textbf{SST-5} & \textbf{MR} & \textbf{CR} & \textbf{MPQA} & \textbf{Subj} & \textbf{TREC} & \textbf{AG News} \\ 
        \midrule
        $\ksgd$ & 16 & $0.39_{(0.14)}$ & 0.70$_{(0.35)}$ & 0.14$_{(0.09)}$ & 0.32$_{(0.03)}$ &     0.56$_{(0.12)}$ & 0.60$_{(0.31)}$ & 2.87$_{(1.27)}$ & 3.52$_{(4.44)}$ \\
        & 64 & $0.66_{(0.31)}$ & 0.97$_{(0.55)}$ & 0.37$_{(0.18)}$ & 0.66$_{(0.43)}$ &     0.44$_{(0.09)}$ & 1.04$_{(0.19)}$ & 9.63$_{(13.36)}$ & 1.74$_{(0.60)}$ \\
        \midrule
        $\ksigngd$ & 16 & $0.45_{(0.11)}$ & 0.61$_{(0.17)}$ & 0.33$_{(0.08)}$ & 0.35$_{(0.13)}$ &     0.48$_{(0.06)}$ & 0.40$_{(0.21)}$ & 1.33$_{(0.14)}$ & 1.50$_{(0.56)}$ \\
        & 64 & $0.34_{(0.09)}$ & 0.77$_{(0.03)}$ & 0.43$_{(0.08)}$ & 0.36$_{(0.04)}$ &     0.50$_{(0.17)}$ & 0.54$_{(0.07)}$ & 1.38$_{(0.12)}$ & 1.44$_{(0.15)}$ \\
        \bottomrule
        \end{tabular}
    }
        \caption{Single-sentence tasks. \vspace{1.0em}}
    \end{subtable}
    \begin{subtable}[h]{1.0\textwidth}
        \centering
    \resizebox{0.76\textwidth}{!}{
        \begin{tabular}{lr|cccccc}
        \toprule
        Method & $k$-shot  & \textbf{MNLI} & \textbf{SNLI} & \textbf{QNLI} & \textbf{RTE} & \textbf{MRPC} & \textbf{QQP} \\ 
        \midrule
        $\ksgd$ & 16 & 1.26$_{(0.20)}$ & 0.58$_{(0.17)}$ & 0.67$_{(0.14)}$ & 0.40$_{(0.25)}$ & 0.65$_{(0.32)}$ & 0.79$_{(0.39)}$ \\
        & 64 & 1.62$_{(0.19)}$ & 0.75$_{(0.04)}$ & 0.89$_{(0.42)}$ & 1.04$_{(0.16)}$ & 1.41$_{(0.53)}$ & 1.00$_{(0.14)}$ \\
        \midrule
        $\ksigngd$ & 16 & 0.52$_{(0.09)}$ & 0.68$_{(0.16)}$ & 0.47$_{(0.09)}$ & 0.48$_{(0.13)}$ & 0.48$_{(0.07)}$ & 0.58$_{(0.07)}$ \\
        & 64 & 0.59$_{(0.03)}$ & 0.62$_{(0.04)}$ & 0.55$_{(0.04)}$ & 0.54$_{(0.02)}$ & 0.60$_{(0.08)}$ & 0.56$_{(0.02)}$ \\
        \bottomrule
        \end{tabular}
    }
        \caption{Sentence-pair tasks.}
    \end{subtable}

    \caption{Average element-wise relative distance of $\ksgd$ and $\ksigngd$ computed on the pre-trained and best model fine-tuned with SGD and Adam respectively. A smaller value indicates a higher likelihood that the Fixed Features property of kernel behavior (\Cref{def:kernel_regime}) holds when performing fine-tuning. Distances are averaged across 5 seeds for each value of $k$ and measured on the held-out test set.
        }
    \label{tab:linearization_kernel_distance}
\end{table}

\subsection{Solvable Task Experiments}\label{sec:app_solvable_task}
We run a preliminary empirical test to verify if various tasks are solvable in the infinite-width limit (see \Cref{def:solvable_task}). 
Intuitively, the assumption states that wider models (with all other architecture and pre-training hyperparameters fixed) will solve the downstream task better in a zero-shot fashion, and in the limit, an infinitely wide model will solve the task perfectly. 
The cheap empirical test involves measuring the average output derivative $\chi$ of the loss w.r.t. the model output (see \Cref{def:output_derivative} for a definition of $\chi$) over the entire dataset for two models of different widths.
We note that our paper uses RoBERTa-base ($n=768$) for experiments, so a natural choice for a wider model would be RoBERTa-large ($n=1024$).
However, RoBERTa-large is also deeper than RoBERTa-base, and indeed, in general, it is difficult to find two publicly available pre-trained models with different widths and fixed depth. 
We nevertheless present the table of $\chi$ values for several downstream tasks measured on RoBERTa-base and RoBERTa-large in \Cref{tab:chi_measurement}.
\begin{table}[h]
\centering
\resizebox{0.9\textwidth}{!}{
\centering
    \begin{tabular}{lcccccccccccccccc}
    \toprule
     Model size  &  \multicolumn{1}{c}{\textbf{SST-2}} & \multicolumn{1}{c}{\textbf{MR}} & \multicolumn{1}{c}{\textbf{CR}} & \multicolumn{1}{c}{\textbf{MPQA}} & \multicolumn{1}{c}{\textbf{Subj}} & \multicolumn{1}{c}{\textbf{QNLI}} & \multicolumn{1}{c}{\textbf{RTE}} & \multicolumn{1}{c}{\textbf{MRPC}} & \multicolumn{1}{c}{\textbf{QQP}} \\
    \midrule
    Base ($n=768$)  & 0.32    & 0.32  & 0.26 & 0.38 & 0.43 & 0.48 & 0.48 & 0.56  & 0.49 \\
    Large ($n=1024$)  & 0.32  & 0.25  & 0.25 & 0.40 & 0.46 & 0.48 & 0.47 & 0.52  & 0.52 \\
    \bottomrule
    \end{tabular}

    }
    \caption{We measure the average output derivative (\Cref{def:output_derivative}) in the prompt-based FT setting for RoBERTa-base and RoBERTa-large. 
    } 
    \label{tab:chi_measurement}
\end{table}

\subsection{Robustness to Choice of Prompt} \label{sec:app_prompt_choice}
We explore different choices of prompt and label words in Table \ref{tab:prompt_format}. When using the results of the prompt and label search from \citet{gao2020making}, we find that the kernel approximation matches fine-tuning well,. However, the choice of prompt does matter and $\ksgd$ performs poorly with the minimal  ``null prompts'' from \citet{logan2022cutting} on sentiment classification datasets, where the prompt is merely ``{\sent} {\mask}'' and the label words remain \{great, terrible\}. We hypothesize this failure is because the task is no longer solvable in the infinite width limit (\Cref{def:solvable_task}).

\begin{table}[!h]
\centering

\resizebox{1.0\textwidth}{!}{
\begin{tabular}{rllcccccc}
\toprule
$k$-shot & Prompt + label format & Method & \textbf{SST-2} & \textbf{MR} & \textbf{CR} & \textbf{QNLI} & \textbf{RTE} & \textbf{QQP} \\
\midrule
\multirow[t]{9}{*}{16} & \multirow[b]{2}{*}{\shortstack[l]{Manual \\ \citep{gao2020making}}} & Adam-FT & 88.3$_{(1.2)}$ & 81.3$_{(6.1)}$ & 93.0$_{(1.6)}$ & 63.1$_{(3.5)}$ & 57.6$_{(6.3)}$ & 61.8$_{(4.5)}$ \\
\cmidrule{3-9}
 &  & SGD-FT & 89.0$_{(1.5)}$ & 83.2$_{(2.4)}$ & 93.3$_{(0.2)}$ & 62.1$_{(3.1)}$ & 60.0$_{(5.5)}$ & 62.1$_{(2.3)}$ \\
 &  & {\cellcolor{gblue!10}$\ksgd$} & {\cellcolor{gblue!10}88.3$_{(0.3)}$} & {\cellcolor{gblue!10}84.7$_{(1.5)}$} & {\cellcolor{gblue!10}93.2$_{(0.9)}$} & {\cellcolor{gblue!10}60.1$_{(3.3)}$} & {\cellcolor{gblue!10}60.0$_{(4.7)}$} & {\cellcolor{gblue!10}58.2$_{(0.9)}$} \\
 \cmidrule{2-9}
 & \multirow[b]{2}{*}{\shortstack[l]{Prompt + label search \\ \citep{gao2020making}}} & Adam-FT & 88.1$_{(0.8)}$ & 81.6$_{(3.8)}$ & 92.8$_{(0.4)}$ & 56.3$_{(3.8)}$ & 58.6$_{(4.6)}$ & 58.6$_{(4.5)}$ \\
 \cmidrule{3-9}
 &  & SGD-FT & 89.2$_{(1.2)}$ & 80.1$_{(1.8)}$ & 93.2$_{(0.5)}$ & 58.7$_{(4.8)}$ & 61.6$_{(2.6)}$ & 59.0$_{(1.4)}$ \\
 &  & {\cellcolor{gblue!10}$\ksgd$} & {\cellcolor{gblue!10}88.6$_{(1.1)}$} & {\cellcolor{gblue!10}78.5$_{(1.2)}$} & {\cellcolor{gblue!10}93.5$_{(0.7)}$} & {\cellcolor{gblue!10}56.7$_{(1.7)}$} & {\cellcolor{gblue!10}57.4$_{(5.5)}$} & {\cellcolor{gblue!10}60.2$_{(2.0)}$} \\
 \cmidrule{2-9}
 & \multirow[b]{2}{*}{\shortstack[l]{Null prompts \\ \citep{logan2022cutting}}} & Adam-FT & 87.6$_{(0.9)}$ & 82.6$_{(0.6)}$ & 92.8$_{(0.6)}$ & 59.0$_{(2.9)}$ & 56.4$_{(4.7)}$ & 57.5$_{(5.2)}$ \\
 \cmidrule{3-9}
 &  & SGD-FT & 88.1$_{(0.7)}$ & 82.8$_{(3.6)}$ & 93.4$_{(0.7)}$ & 59.0$_{(3.4)}$ & 54.1$_{(1.6)}$ & 57.6$_{(5.5)}$ \\
 &  & {\cellcolor{gblue!10}$\ksgd$} & {\cellcolor{gblue!10}78.3$_{(4.3)}$} & {\cellcolor{gblue!10}78.7$_{(1.8)}$} & {\cellcolor{gblue!10}91.7$_{(0.8)}$} & {\cellcolor{gblue!10}55.8$_{(2.7)}$} & {\cellcolor{gblue!10}55.5$_{(2.3)}$} & {\cellcolor{gblue!10}57.4$_{(1.8)}$} \\
\bottomrule
\end{tabular}
}

\caption{We experiment with different prompt formats and label words: using the top result of an automatic prompt search performed on RoBERTa-large (Table E.1 in \citet{gao2020making}); and minimal null prompts (Table A3, \citet{logan2022cutting}), which add no additional text to the prompt. We find that our observations are robust to the choice of prompt, with the exception of the more unnatural ``null prompts'' on sentiment tasks (SST-2, MR, CR), which show a substantial gap between $\ksgd$ and fine-tuning.
We report F1 for QQP and accuracy otherwise, and average the metrics over 5 seeds.}

\label{tab:prompt_format}
\end{table}

\begin{table}[t]
    \centering

    \begin{subtable}[h]{1.0\textwidth}
        \centering
        
        \resizebox{1.0\textwidth}{!}{  
        
        \begin{tabular}{rlcccccccccccccc}
        \toprule
        $k$-shot & Method & \tf{SST-2} & \tf{SST-5} & \tf{MR} & \tf{CR} & \tf{MPQA} & \tf{Subj} & \tf{TREC} & \tf{AG News} \\
        \midrule
            16 & SignGD-FT 
                & 87.6$_{(3.6)}$ & 43.4$_{(3.9)}$ & 84.4$_{(1.1)}$ & 92.8$_{(1.4)}$ & 82.4$_{(1.5)}$ & {90.3}$_{(1.8)}$ & {85.4}$_{(4.0)}$ & {85.2}$_{(1.4)}$ \\
            & Adam-FT 
                & {88.3}$_{(1.2)}$ & {45.4}$_{(2.6)}$ & 81.3$_{(6.1)}$ & 93.0$_{(1.6)}$ & {82.8}$_{(2.2)}$ & 87.4$_{(2.1)}$ & 79.6$_{(6.1)}$ & 84.0$_{(1.6)}$ \\
            & {\cellcolor{ggreen!7}$\ksigngd$} 
                & {\cellcolor{ggreen!7}{88.3}$_{(0.5)}$} & {\cellcolor{ggreen!7}42.2$_{(3.9)}$} & {\cellcolor{ggreen!7}84.3$_{(1.5)}$} & {\cellcolor{ggreen!7}{93.7}$_{(0.5)}$} & {\cellcolor{ggreen!7}76.7$_{(3.3)}$} & {\cellcolor{ggreen!7}{89.2}$_{(2.0)}$} & {\cellcolor{ggreen!7}58.1$_{(6.5)}$} & {\cellcolor{ggreen!7}82.3$_{(1.6)}$} \\
            & {\cellcolor{ggreen!7}$\kasigngd$}
                & {\cellcolor{ggreen!7}{88.3}$_{(0.4)}$} & {\cellcolor{ggreen!7}43.7$_{(1.7)}$} & {\cellcolor{ggreen!7}{84.9}$_{(1.1)}$} & 	{\cellcolor{ggreen!7}93.4$_{(0.5)}$} & {\cellcolor{ggreen!7}74.6$_{(3.5)}$} & {\cellcolor{ggreen!7}88.6$_{(1.8)}$} & {\cellcolor{ggreen!7}22.7${{}}_{(2.8)}$} & {\cellcolor{ggreen!7}83.6$_{(1.0)}$}  \\
        \cmidrule{1-10} 
            64 & SignGD-FT
                & 87.6$_{(2.5)}$ & 47.3$_{(2.7)}$ & {86.2}$_{(1.2)}$ & 93.7$_{(1.7)}$ & {85.3}$_{(1.7)}$ & 92.1$_{(2.0)}$ & {93.7}$_{(0.5)}$ & {87.5}$_{(0.6)}$ \\
            & Adam-FT 
                & {89.3}$_{(0.7)}$ & 48.5$_{(2.0)}$ & 86.0$_{(0.4)}$ & {93.7}$_{(0.8)}$ & 84.6$_{(0.9)}$ & {92.7}$_{(0.6)}$ & 92.6$_{(1.3)}$ & 86.8$_{(1.1)}$ \\
            & {\cellcolor{ggreen!7}$\ksigngd$}
                & {\cellcolor{ggreen!7}89.1$_{(0.5)}$} & {\cellcolor{ggreen!7}{49.1$_{(1.6)}$}} & {\cellcolor{ggreen!7}85.6$_{(1.0)}$} &
                {\cellcolor{ggreen!7}93.9$_{(0.2)}$} &
                {\cellcolor{ggreen!7}79.0$_{(5.8)}$} & {\cellcolor{ggreen!7}92.4$_{(0.5)}$} & {\cellcolor{ggreen!7}82.0$_{(1.4)}$} & {\cellcolor{ggreen!7}85.9$_{(0.7)}$} \\
            & {\cellcolor{ggreen!7}$\kasigngd$}
                & {\cellcolor{ggreen!7}88.9$_{(0.9)}$} &  {\cellcolor{ggreen!7}43.6$_{(2.2)}$} & {\cellcolor{ggreen!7}85.6$_{(1.0)}$} & {\cellcolor{ggreen!7}{94.0}$_{(0.3)}$} & {\cellcolor{ggreen!7}81.8$_{(1.1)}$} & {\cellcolor{ggreen!7}91.8$_{(1.1)}$} &  	{\cellcolor{ggreen!7}21.0$_{(4.3)}$} & {\cellcolor{ggreen!7}86.2$_{(0.3)}$} \\
        \bottomrule
        \end{tabular}

        }
        
        \caption{Single-sentence tasks \vspace{1.0em}}
    \end{subtable}
    \begin{subtable}[h]{1.0\textwidth}
        \centering

        \resizebox{.8\textwidth}{!}{  
\begin{tabular}{rlcccccccccccc}
\toprule
$k$-shot & Method & \tf{MNLI} & \tf{SNLI} & \tf{QNLI} & \tf{RTE} & \tf{MRPC} & \tf{QQP} \\
\midrule
    16 & SignGD-FT
        & {62.1}$_{(4.1)}$ & {67.7}$_{(2.7)}$ & {64.0}$_{(4.8)}$ & {60.9}$_{(5.8)}$ & {78.4}$_{(3.0)}$ & {66.2}$_{(2.1)}$ \\
    & Adam-FT
        & 56.8$_{(2.9)}$ & 64.6$_{(4.1)}$ & 63.1$_{(3.5)}$ & 57.6$_{(6.3)}$ & 77.6$_{(3.1)}$ & 61.8$_{(4.5)}$ \\
    & {\cellcolor{ggreen!7}$\ksigngd$}
        & {\cellcolor{ggreen!7}53.8$_{(1.2)}$} & {\cellcolor{ggreen!7}54.9$_{(2.7)}$} %
        & {\cellcolor{ggreen!7}59.5$_{(3.1)}$} & {\cellcolor{ggreen!7}55.4$_{(4.2)}$} & {\cellcolor{ggreen!7}75.6$_{(1.2)}$} & {\cellcolor{ggreen!7}60.7$_{(2.2)}$} \\
    & {\cellcolor{ggreen!7}$\kasigngd$}
        & {\cellcolor{ggreen!7}51.9$_{(4.0)}$} & {\cellcolor{ggreen!7}54.9$_{(3.1)}$} & {\cellcolor{ggreen!7}56.0$_{(1.9)}$} & {\cellcolor{ggreen!7}59.8$_{(4.0)}$} & {\cellcolor{ggreen!7}75.2$_{(2.6)}$} & {\cellcolor{ggreen!7}59.4$_{(2.0)}$} \\
\cmidrule{1-8} 
    64 & SignGD-FT 
        & {69.3}$_{(1.2)}$ & {77.4}$_{(1.0)}$ & {76.8}$_{(2.2)}$ & 66.4$_{(2.9)}$ & {84.1}$_{(1.3)}$ & {69.9}$_{(0.8)}$ \\
    & Adam-FT
        & 67.9$_{(1.0)}$ & 76.9$_{(1.4)}$ & 74.2$_{(3.2)}$ & {67.3}$_{(2.7)}$ & 80.9$_{(1.2)}$ & 69.8$_{(0.6)}$ \\
    & {\cellcolor{ggreen!7}$\ksigngd$}
        & {\cellcolor{ggreen!7}60.8$_{(1.7)}$} %
        & {\cellcolor{ggreen!7}64.1$_{(2.3)}$} %
        & {\cellcolor{ggreen!7}65.4$_{(1.7)}$} & {\cellcolor{ggreen!7}63.8$_{(1.8)}$} & {\cellcolor{ggreen!7}77.4$_{(2.3)}$} & {\cellcolor{ggreen!7}63.7$_{(4.4)}$} \\
    & {\cellcolor{ggreen!7}$\kasigngd$}
        & {\cellcolor{ggreen!7}58.5$_{(1.7)}$} & {\cellcolor{ggreen!7}66.8$_{(1.1)}$} & {\cellcolor{ggreen!7}66.5$_{(1.1)}$} & {\cellcolor{ggreen!7}63.8$_{(2.2)}$} & {\cellcolor{ggreen!7}77.3$_{(2.0)}$} & {\cellcolor{ggreen!7}66.1$_{(3.4)}$} \\
\bottomrule
\end{tabular}
        }
        
        \caption{Sentence-pair tasks \vspace{-0.5em}}     
    \end{subtable}

    \caption{Comparing the performnace SignGD-FT to Adam-FT, $\ksigngd$ and $\kasigngd$ in the prompt-based setting on the LM-BFF test set \citep{gao2020making}. SignGD fine-tuning applies the $\sign$ function coordinate-wise to gradients before taking gradient steps, and leads to surprisingly strong results, especially on sentence-pair tasks. We search over the same hyperparameter gird as for Adam-FT, see Table \ref{tab:hyperparameters}, and we do not use momentum. 
    Performance is measure by average test accuracy over 5 $k$-shot splits for all tasks except MRPC and QQP, where it is F1.}
    \label{tab:signgd_ft}
\end{table}

\clearpage
\section{Kernel Behavior and the Parametrization}\label{sec:app_kernel_theory}
Neural network training can exhibit either kernel behavior or feature learning behavior.
These were described in~\cite{woodworth2019kernel} as the lazy regime and active regime, respectively, when training from a random initialization.
Kernel behavior provides a tractable tool to study the training of neural networks, but it is not believed to be a complete description of practical deep learning settings. 
In particular, kernel behavior implies the feature (i.e., gradient) of the neural networks remains unchanged in the overparameterized setting, which is not true in practical pre-training of large models. 

\cite{yang2021tensor4} showed how the initialization variance, multiplier, and learning rate for each parameter can move training from the kernel behavior to the feature learning behavior. 
They further developed the Maximal Update Parametrization (abbreviated MUP or $\mup$) where every parameter is updated maximally (in terms of scaling with width) while keeping the network stable. 
\cite{yang2022tensor5} then extends $\mup$ to Transformers with Adam optimization, and showed empirically that for pre-training of large language models using $\mup$, the optimal hyperparameters remain the same when increasing width. 
It allows more comprehensive hyperparameter searches on a smaller model and direct transfer of the resulting optimal hyperparameters to the larger model, resulting in markedly improved pre-training performance.

This section discusses two of our formal results: \Cref{thm:prompt_ft_kernel_regime,thm:asigngd_kernel}.
In general, we consider the overparameterized setting in which the width of the network goes to infinity. 
Additionally, we assume that when initializing a weight matrix of the model, each entry of the matrix is drawn from i.i.d. Gaussian distribution. 
In particular, we model a pre-trained model as a non-random initialization that arose from training starting at a random initialization.
We use Tensor Programs~\citep{yang2020tensor3} for our formal results. 

This section is organized as follows.
In \Cref{sec:app_tp_prelims}, we introduce the basic notation and ideas around Tensor Programs as well as the assumptions we need to make in order for an infinite-width limit to be interesting to study.
Then, \Cref{sec:theory_signgd_kernel} gives the formal proof for the kernel analog to SignGD (\Cref{thm:asigngd_kernel}).
In \Cref{sec:theory_prompt_finetuning}, we provide a formal proof of how fine-tuning can exhibit kernel behavior (\Cref{thm:prompt_ft_kernel_regime}). The proof relies heavily on Tensor Programs, so we additionally provide a more accessible and intuitive sketch on linear networks in \Cref{sec:app_linear_intuition}.

\subsection{Preliminaries}\label{sec:app_tp_prelims}
\paragraph{Notations}

Let $\xi\in \R^{d_{in}}$ be the input of the network. Let $n$ be the hidden dimension of the network and $d_{out}$ be the output dimension of the network. 
We define the network as a function of the following form:
\[f(\xi;\{U^i\}_i, \{W^j\}_j, V) = V^\top h(\xi; \{U^i\}_i, \{W^j\}_j),\]
where $\xi$ is the input, $U^i\in \R^{n\times d_{in}}$ are the input weight matrices, $W^j\in \R^{n\times n}$ are hidden weight matrices, $V\in \R^{n\times d_{out}}$ is the output weight matrix, and $h(\xi; \{U^i\}_i, \{W^j\}_j)\in \R^n$ is the input of last layer (readout layer). \footnote{We are able to describe transformers (without weight tying) in the definition. %
The bias can be regarded as input weights assuming there is a coordinate in $\xi$ that is always 1.} 
We write $\gM$ as the set of weight matrices, i.e., $\gM=\{U^i\}_i\cup \{W^j\}_j\cup\{V\}$.
For $M\in \gM$, let $\nabla_M f(\xi)$ be the gradient of $f$ w.r.t. $M$ at input $\xi$.

To simplify the notation, we assume $d_{in}=1$ in this section. 
We will note when an extension to $d_{in}>1$ requires a non-trivial step.
For any weight matrix $M\in \gM$, 
let $\gamma_M$ be the multiplier of $M$, such that $M$ is multiplied by $\gamma_M$ before performing matrix multiplication.
Let $\eta_M$ be the learning rate of the weight $M$.
Let $\sigma_M^2$ be the variance of entries of $M$ at initialization, so each entry of $M$ is drawn $\gN(0, \sigma_M^2)$ independently. Since our focus is the prompt-based fine-tuning, we assume no change is made to the network at the beginning of fine-tuning, and the learning rates for pre-training and fine-tuning are the same unless otherwise noted.

Because we are considering the infinite-width limit, $f(\xi;\{U^i\}_i, \{W^j\}_j, V)$ actually represents a series of increasingly wide networks $\{f^n(\xi; \{U^{i,n}\}_i, \{W^{j,n}\}_j, V^n)\}_{n>0}$ of the same architecture, but $f^n$ has a hidden dimension $n$. We use the notation $f$ to include the model architecture, the training optimizer of the model, and $\gamma_M, \eta_M, \sigma_M$ for every weight matrix $M$ in the model.

Let $M_t$ be the weight matrix at time step $t$ of training. If the network is pre-trained, we let $M_{-1}$ be the weight matrix before pre-training, and $M_0$ be the parameters right after pre-training. 
Let $\Delta M_t = M_t - M_{t-1}$ be the change each training step induces. %
Let $f_t$ be the network at step $t$ that 
\[f_t(\xi) = f(\xi;\{U^i_t\}_i, \{W^j_t\}_j, V_t).\]

Let $\xi_t, y_t$ be the training input and target at step $t$, and let the loss function at step $t$ be $\ell(f_{t-1}(\xi_t), y_t)$. For ease of notation, we often absorb $y_t$ into $\ell$ and denote $\ell_t(f_{t-1}(\xi_t))\triangleq \ell(f_{t-1}(\xi_t), y_t)$. Let $\chi_t = \ell_t'(f_{t-1}(\xi_t))$ be the derivative of the loss function, as defined in \Cref{def:output_derivative}. We assume $\ell_t''$ (second derivative of $\ell_t$) is bounded\footnote{For $C$-way classification, the assumption is extended to its multivariate version: each entry of Hessian of $\ell_t$ is bounded.}, which is satisfied when $\ell$ is mean square loss or cross entropy loss.

\paragraph{Big-O Notation}
For a series of scalar random variables $c=\{c^n\}_{n>0}$ and a function $e:\sN\to \R$, we say $c=\Theta(e(n))$ if there exist $A, B$ such that for sufficiently large $n$, $|c^n|\in [Ae(n),Be(n)]$ almost surely. For a series of vector random variables $x=\{x^n\}_{n>0}$, we say that $x$ is coordinate-wise $\Theta(n^a)$, or $x=\Theta(e(n))$ if this series of scalar random variables $\left\{\|x^n\|_2/\sqrt{n}\right\}_{n>0}$ is $\Theta(e(n))$. Similarly for the notation $O(e(n))$, $\Omega(e(n))$, and $o(e(n))$. For convenience, we assume every $e(n)$ in this section is equal to $n^a$ for some $a$.

\paragraph{Tensor Programs}
We refer reader to see Section 7 of \cite{yang2021tensor4} for detailed explanation and full definition of Tensor Programs. Here, we provide a simple overview of Tensor Programs: \begin{definition}[Definition 7.1 of \cite{yang2021tensor4}]
A Tensor Program is a sequence of $\R^n$-vectors and $\R$-scalars inductively generated
via one of the following ways from an initial set $\gC$ of random scalars, $\gV$ of random $\R^n$ vectors, and a set $\gW$ of random $\R^{n\times n}$ matrices.
\begin{itemize}[leftmargin=.7in]
    \item[\textbf{MatMul}] Given $W\in \R^{n\times n}$ and $x\in \R^n$, we can generate $Wx\in \R^n$ or $W^\top x\in \R^n$.
    \item[\textbf{Nonlin}] Given $\phi:\R^k\times \R^l\to \R$, previous scalar $\theta_1,\ldots, \theta_l\in \R$ and vector $x^1,\ldots, x^k\in \R^n$, we can generate a new vector \[\phi(x^1, \ldots, x^k; \theta_1,\ldots, \theta_l)\in \R^n\]
    where $\phi(-;\theta_1,\ldots, \theta_l)$ applies coordinate-wise to each ``$\alpha$-slice '' $(x_\alpha^1,\ldots,x_\alpha^k)$.
    \item[\textbf{Moment}] Given the same setup as above, we can also generate a new scalar \[\frac1n\sum_{\alpha=1}^n \phi(x^1_\alpha, \ldots, x^k_\alpha; \theta_1,\ldots, \theta_l)\in \R. \]
\end{itemize}
\end{definition}
\citet{yang2019wide,yang2020tensor2,yang2021tensor2b,yang2022tensor5} show that Tensor Programs can express the computation, SGD/Adam optimization, and the kernel of almost any general architecture.

The key result of the Tensor Programs is that we can represent the coordinates of any vector $x$ in the Tensor Program with a random variable $Z^x$, and represent any scalar $\theta$ with a deterministic scalar $\mathring\theta$. There is a way to define all $\mathring\theta$ and $Z^x$ correspond to the Tensor Program (cf. Definition 7.3 in \cite{yang2021tensor4}), and the Master Theorem of the Tensor Program shows that $\theta \to \mathring\theta$ when $n\to \infty$ (cf. Theorem 7.4 in \cite{yang2021tensor4}). 

Although it is in general hard to compute $Z^x$ and $\mathring\theta$, it allows us to reason about the scales of vectors in the training of a network.

\paragraph{Assumptions Related to Tensor Programs.} Since we are studying the infinite width limit and using Tensor Programs as our framework, there are some mild assumptions that we need in order to apply Tensor Programs and results in \cite{yang2021tensor4}.

\begin{assumption}\label{assump:tp_nontrivial_stable}
We assume the network $f$ satisfies the following
\begin{enumerate}[label=\alph*)]
    \item The forward pass of $f$ in the infinite-width limit can be written as Tensor Programs.
    \item The hidden vectors have $\Theta(1)$ coordinates at initialization.
    \item The hidden vectors have $O(1)$ coordinates during training.
    \item For any training scheme\footnote{Training scheme means a sequence of training examples $\{(\xi_t, y_t\}_{t>0}$, and loss function $\ell(f_t(\xi_t),y_t)$.} and any constant $t$ and any input $\xi$, $f_t(\xi)=O(1)$.
    \item There exist a training scheme and some constant $t$ and input $\xi$ such that $f_t(\xi)-f_0(\xi)=\Theta(1)$.
    \item The activation function of $f$ is tanh or $\sigma$-gelu for a small enough $\sigma$ (so it approximates ReLU), where 
    \[\sigma\text{-}gelu(x) = \frac12x\mathrm{erf}(\sigma^{-1}x)+\sigma\frac{e^{-\sigma^{-2}x^2}}{2\sqrt \pi}+\frac x 2.\] 
\end{enumerate}
Furthermore, we have two assumption on SignGD: 

\begin{enumerate}
\item[g)] SignGD is approximated as the $\sign$ function being replaced with $\epssign$ for small enough $\eps$ when updating parameters, where $\epssign(x)=\frac{x}{|x|+\eps}$ is smoothed version of $\sign$. We assume using different $\eps$ when computing the sign of $\nabla_M f$, so that $\eps$ for $\nabla_M f$ match the maximum scale of $\nabla_M f$.
\item[h)] The ratio between the learning rate of SignGD in prompt-based fine-tuning and the learning rate of pre-training matches the maximum $\chi$ after pre-training. That is, we assume $\eta_M= \Theta(\eta_M^{\text{PT}} \cdot \chi_{\max})$ where $\eta_M^{\text{PT}}$ is learning rate of pre-training for SignGD, and $\chi_{\max} =\max_{(\xi, y)\in \Xi} \chi(\xi, y, f_0)$.
\end{enumerate}
\end{assumption}

b), c), d) and e) in \Cref{assump:tp_nontrivial_stable} together recover the definition of nontrivial stable network in \cite{yang2021tensor4}. b) and c) ensure that the pre-activations in the network are not too large, so that activation functions (e.g., tanh) are not trivialized to always output $\pm 1$.  b) ensures that the pre-activations in the network are not too small at initialization, so the activation function is not trivialized to its first-order Taylor expansion. d) ensures the network output is bounded. e) ensures that the network is not frozen during training (i.e., learning can occur).

f) and g) in \Cref{assump:tp_nontrivial_stable} assures all non-linear functions that appear in the Tensor Programs is pseudo-Lipschitz, which is required for the Master Theorem of Tensor Programs. g) also assures that $\epssign$ is not trivialize to $0$ or $\sign$ when $\nabla_M f\neq \Theta(1)$.

h) in \Cref{assump:tp_nontrivial_stable} assures when $\chi=o(1)$, updates of SignGD in fine-tuning is not of bigger scale than SGD. It is also observed in practice that the optimal learning rate for fine-tuning is smaller than the learning rate for pre-training.

\subsection{SignGD Kernel Derivation}\label{sec:theory_signgd_kernel}

\begin{definition}[Formal Definition of Kernel Behavior]\label{def:theory_kernel_behavior}

	We say that this network training process demonstrates \textit{kernel behavior} if the following properties are satisfied.
	\begin{enumerate}
   		\item \textit{Linearization}: The change of the network can be approximated by its first order Taylor expansion, i.e., 
\[\lim_{n\to \infty}\frac{f_t(\xi)-f_{t-1}(\xi)}{\chi_{\max}}=\lim_{n\to\infty} \sum_{M\in \gM} \left\langle \nabla_{M} f_{t-1}(\xi),\frac{\Delta M_t}{\chi_{\max}} \right\rangle;\] where $\chi_{\max}=\max_{(\xi, y)\in \Xi} \chi(\xi, y, f_0)$, $\Xi$ is the training dataset.
		\item \textit{Fixed Features}: The gradients at step $t$ are approximately the same as before training, i.e.,
\[\forall M\in \gM,  \lim_{n\to\infty} \frac{\|\nabla_M f_t(\xi) - \nabla_M f_0(\xi)\|^2_2}{\max_{\xi'} \|\nabla_{M} f_0(\xi')\|^2_2}=0.\]
	\end{enumerate}
\end{definition}
Note that we define Linearization with both LHS and RHS divided by $\chi_{\max}$ so it is meaningful for the case of $\chi=o(1)$. We do the same thing in the following theorem.
\begin{theorem}[SignGD Kernel]\label{thm:theory_signgd_kernel}
    If SignGD training of $f$ demonstrates kernel behavior, then under \Cref{assump:tp_nontrivial_stable},
    \[\lim_{n\to \infty} \frac{f_t(\xi)-f_{t-1}(\xi)}{\chi_{\max}}=\lim_{n\to\infty} \sum_{M\in \gM} -\tilde\eta_M\left\langle \nabla_{M} f_0(\xi),\epssign(\nabla_{M} f_0(\xi_t)) \right\rangle,\]
    where $\tilde\eta_M=\eta_M\sign(\chi_t)/\chi_{\max}$.
\end{theorem}
Note if $\eta_M=\eta$, the RHS of the equation above equals to 
\[  -\frac{\eta\sign(\chi_t)}{\chi_{\max}}\langle \nabla f_0(\xi), \epssign(\nabla f_0(\xi_t)\rangle\approx-\frac{\eta\sign(\chi_t)}{\chi_{\max}}\kasigngd(\xi, \xi_t),\]
where the approximation comes from the difference between $\epssign$ and $\sign$.
\begin{proof}
By the update rule of SignGD, $\frac{\Delta M_t}{\chi_{\max}}=-\tilde\eta_M\epssign(\nabla_M f_{t-1})$. It suffices to prove \[\tilde\eta_M \left\langle \nabla_{M} f_t(\xi),\epssign(\nabla_{M} f_t(\xi_t)) \right\rangle=\tilde\eta_M\left\langle \nabla_{M} f_0(\xi),\epssign(\nabla_{M} f_0(\xi_t)) \right\rangle\] when $n\to \infty$.

Since 
\begin{align}
    &~\tilde\eta_M \left\langle \nabla_{M} f_t(\xi),\epssign(\nabla_{M} f_t(\xi_t)) \right\rangle-\tilde\eta_M\left\langle \nabla_{M} f_0(\xi),\epssign(\nabla_{M} f_0(\xi_t)) \right\rangle \nonumber\\
    =&~ \tilde\eta_M \left\langle \nabla_{M} f_t(\xi)-\nabla_{M} f_0(\xi),\epssign(\nabla_{M} f_t(\xi_t)) \right\rangle+\label{eq:signgd_1}\\
    &~\tilde\eta_M \left\langle \nabla_{M} f_t(\xi),\epssign(\nabla_{M} f_t(\xi_t))-\epssign(\nabla_{M} f_0(\xi_t)) \right\rangle+\label{eq:signgd_2}\\
    &~\tilde\eta_M \left\langle \nabla_{M} f_t(\xi)-\nabla_{M} f_0(\xi),\epssign(\nabla_{M} f_t(\xi_t))-\epssign(\nabla_{M} f_0(\xi_t)) \right\rangle\label{eq:signgd_3},
\end{align}
we only need to prove \Cref{eq:signgd_1,eq:signgd_2,eq:signgd_3} are all 0 when $n\to \infty$.

Let $\xi^*=\argmax_{\xi'}\|\nabla_{M} f_0(\xi')\|^2_2$ be the input of maximum gradient scale, then by Fixed Features, we have 
\begin{equation}\label{eq:fixed_feature_xi_prime}
    \frac{\|\nabla_M f_t(\xi) - \nabla_M f_0(\xi)\|_2}{\|\nabla_{M} f_0(\xi^*)\|_2}=o(1).
\end{equation}
Since $\epssign(x)-\epssign(y)\leq |x-y|/\eps$,
\begin{align}
    \|\epssign(\nabla_M f_t(\xi)) - \epssign(\nabla_M f_0(\xi))\|_2\leq \|\nabla_M f_t(\xi) - \nabla_M f_0(\xi)\|_2/\eps.\label{eq:signgd_4}
\end{align}
Combined with $\|\nabla_M f_0(\xi^*)\|_2/\sqrt N=\Theta(\eps)$ ($N$ is the number of entries of $M$, this is by g) of \Cref{assump:tp_nontrivial_stable}), we have 
\begin{align}
    &~\frac{\|\epssign(\nabla_M f_t(\xi)) - \epssign(\nabla_M f_0(\xi))\|_2}{\|\epssign(\nabla_{M} f_0(\xi^*))\|_2} \nonumber\\
    \leq &~\frac{ \|\nabla_M f_t(\xi) - \nabla_M f_0(\xi)\|_2/\eps}{\|\epssign(\nabla_{M} f_0(\xi^*))\|_2} & \text{by \cref{eq:signgd_4}}\nonumber\\
    =&~\frac{ \|\nabla_M f_t(\xi) - \nabla_M f_0(\xi)\|_2}{\|\nabla_M f_0(\xi^*)\|_2}\cdot \frac{\|\nabla_M f_0(\xi^*)\|_2/\sqrt N}{\eps \|\epssign(\nabla_{M} f_0(\xi^*))\|_2/\sqrt N}\nonumber\\
    =&~\frac{ \|\nabla_M f_t(\xi) - \nabla_M f_0(\xi)\|_2}{\|\nabla_M f_0(\xi^*)\|_2}\cdot \Theta(1)=o(1). \label{eq:fixed_feature_xi_prime_sign}
\end{align}

By d) in \Cref{assump:tp_nontrivial_stable}, and consider the training scheme that sets $\xi_1=\xi^*$ and the loss function $\ell_t$ so $\chi_1=\Theta(1)$, then
\[\frac{f_1(\xi^*)-f_0(\xi^*)}{\chi_1}=-\frac{\eta_M\sign(\chi_1)}{\chi_1}\left\langle \nabla_{M} f_0(\xi^*),\epssign(\nabla_{M} f_0(\xi^*)) \right\rangle=O(1).\]
By h) in \Cref{assump:tp_nontrivial_stable}, the scale of $\tilde\eta_M$ is identical across different training scheme, so we have 
\[-\tilde\eta_M\left\langle \nabla_{M} f_0(\xi^*),\epssign(\nabla_{M} f_0(\xi^*)) \right\rangle=O(1).\]

And it is easy to see that $\tilde\eta_M\|\nabla_{M} f_0(\xi^*)\|_2\|\epssign(\nabla_{M} f_0(\xi^*))\|_2$ has the same scale as 
$\tilde\eta_M\left\langle \nabla_{M} f_0(\xi^*),\epssign(\nabla_{M} f_0(\xi^*)) \right\rangle$, which is $O(1)$.

Given \Cref{eq:fixed_feature_xi_prime,eq:fixed_feature_xi_prime_sign}, we are about to prove \Cref{eq:signgd_1,eq:signgd_2,eq:signgd_3} divided by $\tilde\eta_M\|\nabla_{M} f_0(\xi^*)\|_2\|\epssign(\nabla_{M} f_0(\xi^*))\|_2$ are all 0 when $n\to\infty$. Provided that $\tilde\eta_M\|\nabla_{M} f_0(\xi^*)\|_2\|\epssign(\nabla_{M} f_0(\xi^*))\|_2=O(1)$, it will imply \Cref{eq:signgd_1,eq:signgd_2,eq:signgd_3} are all 0 when $n\to\infty$, thus conclude our whole proof.

For \Cref{eq:signgd_1},
\begin{align*}
    &~\frac{\tilde\eta_M \left\langle \nabla_{M} f_t(\xi)-\nabla_{M} f_0(\xi),\epssign(\nabla_{M} f_t(\xi_t)) \right\rangle}{\tilde\eta_M\|\nabla_{M} f_0(\xi^*)\|_2\|\epssign(\nabla_{M} f_0(\xi^*))\|_2}\\
    \leq&~ \frac{\|\nabla_{M} f_t(\xi)-\nabla_{M} f_0(\xi)\|_2\|\epssign(\nabla_{M} f_t(\xi_t))\|_2}{\|\nabla_{M} f_0(\xi^*)\|_2\|\epssign(\nabla_{M} f_0(\xi^*))\|_2}\\
    =&~ \frac{\|\nabla_{M} f_t(\xi)-\nabla_{M} f_0(\xi)\|_2}{\|\nabla_{M} f_0(\xi^*)\|_2}=o(1). & \text{by \cref{eq:fixed_feature_xi_prime}}
\end{align*}
Similarly, for \Cref{eq:signgd_2},
\begin{align*}
    &~\frac{\tilde\eta_M \left\langle \nabla_{M} f_t(\xi),\epssign(\nabla_{M} f_t(\xi_t))-\epssign(\nabla_{M} f_0(\xi_t)) \right\rangle}{\tilde\eta_M\|\nabla_{M} f_0(\xi^*)\|_2\|\epssign(\nabla_{M} f_0(\xi^*))\|_2}\\
    \leq &~ \frac{\|\epssign(\nabla_M f_t(\xi)) - \epssign(\nabla_M f_0(\xi))\|_2}{\|\epssign(\nabla_{M} f_0(\xi^*))\|_2}=o(1), & \text{by \cref{eq:fixed_feature_xi_prime_sign}}
\end{align*}
and for  \Cref{eq:signgd_3},
\begin{align*}
    &~\frac{\tilde\eta_M \left\langle \nabla_{M} f_t(\xi)-\nabla_{M} f_0(\xi),\epssign(\nabla_{M} f_t(\xi_t))-\epssign(\nabla_{M} f_0(\xi_t)) \right\rangle}{\tilde\eta_M\|\nabla_{M} f_0(\xi^*)\|_2\|\epssign(\nabla_{M} f_0(\xi^*))\|_2}\\
    \leq &~ \frac{\|\epssign(\nabla_M f_t(\xi)) - \epssign(\nabla_M f_0(\xi))\|_2}{\|\epssign(\nabla_{M} f_0(\xi^*))\|_2}\cdot \frac{\|\nabla_{M} f_t(\xi)-\nabla_{M} f_0(\xi)\|_2}{\|\nabla_{M} f_0(\xi^*)\|_2}\\
    =&~ o(1). & \text{by \cref{eq:fixed_feature_xi_prime,eq:fixed_feature_xi_prime_sign}}
\end{align*}
\end{proof}

\subsection{Prompt-based Fine-Tuning}\label{sec:theory_prompt_finetuning}
Prompt-based fine-tuning uses the pre-trained network directly without substituting or adding any parameters. 
Therefore, without any additional assumptions, the behaviors of fine-tuning and pre-training are the same from the perspective of the Tensor Programs.
We thus adopt the assumption that $\chi = o(1)$ before fine-tuning (\Cref{def:solvable_task}). Without the assumption, the fine-tuning of $f$ will not exhibits kernel behavior if the pre-training is in feature learning regime. 
Intuitively, this assumption is believable because wider pre-trained networks can solve downstream tasks better. 
In this section, we prove that prompt-based fine-tuning exhibits kernel behavior when this assumption holds. 

\begin{theorem}\label{thm:theory_prompt_finetuning}
    If the downstream task $\Xi$ is natural for network $f$, that is, \[\chi_{\max}\triangleq\max_{(\xi, y)\in \Xi} \chi(\xi, y, f_0)=o(1),\] then under \Cref{assump:tp_nontrivial_stable},
    the fine-tuning of $f$ exhibits kernel behavior (\Cref{def:theory_kernel_behavior}).
\end{theorem}
Below we provide a proof that is heavily based on Tensor Programs and the analysis in \cite{yang2021tensor4}. For readers who are not familiar with Tensor Programs, we provide intuitive examples in the next few subsections, where we focus on a three-layer linear network parameterized with $\mup$.
\begin{proof}
The high-level proof consists of two parts: 1) we prove after each step, the update of the function $f$ is $O(\chi_t)$. Combined $\ell_t''$ always bounded by some constant $C$, we can inductively prove $\chi_{t}\leq \chi(\xi_{t}, y_t, f_0) + C \cdot |f_{t-1}(\xi_t) - f_0(\xi_t)| = O(\chi_{\max})$ for all $t$. 2) Given $\chi_{t}= O(\chi_{\max}) = o(1)$, we show the fine-tuning exhibits kernel behavior.

We first prove the theorem under the assumption that the network is a multilayer perceptron and the optimizer is SGD, which is the same setting as \cite{yang2021tensor4}. We will later extend this to more general cases.

Consider the following $L$-hidden-layer perceptron:
\[h^1(\xi)=U\xi,\]
and \[x^l(\xi)=\phi(h^l(\xi)),\quad h^{l+1}(\xi)=W^{l+1}x^l(\xi), \text{ for }l=1,\ldots,L-1,\]
and \[f(\xi)=Vx^L(\xi).\]
Following \cite{yang2021tensor4}, we let the learning rate for every parameter equal to $\eta n^{-c}$. Let $W^1=U$ and $W^{L+1}=V$, and for $l=1,\ldots,L+1$, we parametrize $W^l$ as $W^l=\gamma_l w^l$ for actual trainable parameter $w^l$, and we initialize each coordinate $w^l$ i.i.d. from $\gN(0, \sigma^2_l)$. The setting covers all possible parameterizations based on \Cref{lem:mup_freedom}. For convenience, we assume $\gamma_l=n^{-a_l}$ and $\sigma_l=n^{-b_l}$. Without loss of generality, we further assume that $\chi_{\max}=\Theta(n^{-d})$. Below, we will also inductively show $\chi_t=O(n^{-d})$ by showing $|f_{t+1}-f_t|=O(n^{-d})$.

By Theorem 3.3 of \cite{yang2021tensor4}, stable network implies 
\[r\triangleq \min(a_{L+1}+b_{L+1}, 2a_{L+1}+c)+c-1+\min_{l=1}^L[2a_l+\sI(l=1)]\geq 0.\]
Also by Theorem 3.8 of \cite{yang2021tensor4}, for nontrivial stable network (included in \Cref{assump:tp_nontrivial_stable}), if $r>0$ then there exists a kernel $\gK$ such that \[f_{t+1}(\xi)=f_t(\xi)-\eta \chi_t \gK(\xi,\xi_t),\]
which is very close to our definition of kernel behavior. In fact, we will prove that they are equivalent in the fine-tuning case. 

Since $\chi_t=O(n^{-d})$ for fine-tuning, it is equivalent to set the learning rate to $\eta n^{-c-d}$ and replace $\chi_t$ with $\hat\chi_t=n^d\chi_t=O(1)$.
Formally, we are considering the following training scheme: at the pre-training stage, $r\geq 0$ (so it could demonstrate feature learning or kernel behavior); at the fine-tuning stage, $c$ is increased to $c'\triangleq c+d>c$, thus, the corresponding $r$ is increased to be strictly greater than 0. Therefore, it suggests kernel behavior with following caveats. 

\paragraph{Do we handle the case of different learning rates during pre-training and fine-tuning?}
The answer is \emph{effectively YES}, because the above scheme is equivalent to training from scratch with learning rate $\eta n^{c-d}$. First of all, the scale of the update on $W^l$, $h^l$, $x^l$ and $f$ are all multiplied by $n^{-d}$ when switching from the pre-training stage ($\eta n^{-c}$ learning rate) to the fine-tuning stage($\eta n^{-c-d}$ learning rate). The scales are exactly the same as training from scratch with $\eta n^{-c-d}$ learning rate except $b_{L+1}$ needs to be changed to $b_{L+1}'\triangleq\min(b_{L+1}, a_{L+1}+c)$. Note this change of $b_{L+1}$ does not affect the fact that $r$ is updated to $r'\triangleq r+d>0$.

\paragraph{Does $r'>0$ formally imply our definition of kernel behavior (\Cref{def:theory_kernel_behavior})?} The answer is \emph{YES}. We first prove Fixed Features in \Cref{def:theory_kernel_behavior}. The gradient of matrix $W^l$ is equal to outer product between $\nabla_{h^l} f$ (gradient w.r.t. $h^l$) and $x^{l-1}$. Let $dh^l_t$ be the normalized gradient w.r.t. $h^l$ at step $t$ (so $dh^l_t=\Theta(1)$), and $x^l_t$ be the $x^l$ at step $t$ ($x^l_t=\Theta(1)$ without normalization). It suffices to prove $dh^l_t-dh^l_0=O(1)$ and $x^l_t-x^l_0=o(1)$. The later was proved by Proposition H.27 of \cite{yang2021tensor4}. To prove $dh^l_t-dh^l_0=O(1)$, we let $dx_t^l$ be the the normalized gradient w.r.t. $x^l$ at step $t$, and compute the scale of $dh_t^l-dh_{t-1}^l$ and $dx_t^l-dx_{t-1}^l$ inductively from $l=L$ to $l=1$. We obtain that they both has the same scale of
\[n^{-\min(2a_{L+1}+c-a_{L+1}-b_{L+1}', a_{L+1}+b_{L+1}+c'-1+\min_{m=l+1}^L 2a_m)}\leq n^{-\min(0, r')}=1,\]
the inequality is because $b_{L+1}'\leq a_{L+1}+c$ and $r'\leq a_{L+1}+b_{L+1}+c'-1+\min_{m=l+1}^L 2a_m$.

Second, we prove Linearization in \Cref{def:theory_kernel_behavior}. We need to first make a slight modification to the Tensor Program in \cite{yang2021tensor4}, that is, changing the computation of $f_t(\xi)-f_{t-1}(\xi)$ to $n^d (f_t(\xi)-f_{t-1}(\xi))$. By Theorem H.32 of \cite{yang2021tensor4} and its definition of $\Sigma$, we can show that 
\begin{align*}
    \lim_{n\to \infty} n^{d} (f_t(\xi)-f_{t-1}(\xi))=&~\lim_{n\to\infty} \sum_{l=1}^{L+1} \eta n^{-c}\frac{\chi_t}{n^{-d}}\left\langle \nabla_{W^l} f_{t-1}(\xi),\nabla_{W^l} f_{t-1}(\xi_t) \right\rangle\\
    =&~\lim_{n\to\infty} \sum_{l=1}^{L+1} \left\langle \nabla_{W^l} f_{t-1}(\xi),\frac{\Delta W_t^l}{n^{-d}} \right\rangle.
\end{align*}

This is exactly Linearization in \Cref{def:theory_kernel_behavior} if we multiply $n^{-d} / \chi_{\max}$ on both side. Meanwhile, it also implies $f_t(\xi) - f_{t-1}(\xi)=O(n^{-d})$.

\paragraph{From SGD to SignGD.} Since $\sign(xy)=\sign(x)\sign(y)$, the update of matrix $W^l$ can still be written as outer product of two vectors, i.e., $\Delta W_t^l=\eta n^{-c-d}\sign(\chi_t)\sign(\nabla_{h^l} f_{t-1}) \otimes \sign(x^{l-1}_{t-1})$. After applying $\sign$, the scale of vector changes. If the parametrization is the same, the scales of vectors using SignGD will be different from those using SGD. This can be easily resolved by changing learning rates for each parameter (as in \Cref{assump:tp_nontrivial_stable}), so the scaling change brought by $\sign$ is corrected. Furthermore, as also mentioned in \Cref{assump:tp_nontrivial_stable}, we need to approximate $\sign$ by a smoothed version $\epssign$ so the Master Theorem of Tensor Programs can still apply.

\paragraph{Extension to universal architectures.} The theorem can apply to any network whose first forward pass can be written as Tensor Programs. Given this condition, the forward pass, backward pass, and kernel of any step can be written as Tensor Programs \citep{yang2020tensor2,yang2020tensor3}. To analyse the scaling of the Tensor Program will need the following steps:
\begin{enumerate}
    \item \emph{Extension to general computation graph.} We can still inductively reason about the scale of preactivations and activations by the topological order of the computation graph; and similarly reason about the gradient by the reverse topological order.
    \item \emph{Extension to weight sharing.} We may use weights multiple times in a forward pass. The preactivations, activations and their gradients will not be affected. Only the update of a weight is now a sum of several vector outer product depending on the number of occurrence of the weight.
\end{enumerate}
\end{proof}

\subsection{$\mup$ for SGD and SignGD}
In the following subsections, we provide more intuition for \Cref{thm:theory_prompt_finetuning}. 
Although we consider all types of pre-trained models, we are mostly interested in models with feature learning behavior, because it is likely not true that gradients can be approximated as fixed throughout the entirety of \textit{pre-training}.
For pre-trained models with kernel behavior, it is obvious that fine-tuning with the same settings as pre-training (i.e., prompt-based FT) will also exhibit kernel behavior. 
Furthermore, Theorem H.17 of \cite{yang2021tensor4} proved that if the last layer is replaced with a freshly initialized layer (i.e., standard FT), fine-tuning from a pre-trained models with kernel behavior is the same as training on the downstream task from scratch.

Among all the pre-training schemes that exhibit feature learning behavior, $\mup$ is special because each parameter (except the last layer) can \textit{on its own} push the model to perform feature learning. 
Therefore, to build an intuitive description of fine-tuning behavior, we assume that the model was pre-trained by $\mup$. 
We note again that our main result \textit{does not require} this assumption.

The formulation of $\mup$ contains three sets of hyperparameters: initial variance of $M$, multiplier of $M$ and learning rate of $M$ for $M\in \{U^i\}_i \cup \{W^j\}_j \cup \{V\}$.
However, even if we restrict these three hyperparameters to be in the form of $n^\alpha$, $\mup$ is not unique, because there is one degree of freedom for each weight according to the following lemma.
\begin{lemma}[Lemma J.1 of \cite{yang2022tensor5}]\label{lem:mup_freedom}
Consider a weight matrix $M$ with learning rate $C$, initialized as $M \sim \gN(0, B^2)$, and with a multiplier $A$. Then for any $\gamma > 0$, $f_t(\xi)$ stays fixed for all $t$
and $\xi$ if we set
\begin{itemize} 
    \item $A\gets A\gamma, B\gets B/\gamma, C\gets C/\gamma^2$ if training with SGD.
    \item $A\gets A\gamma, B\gets B/\gamma, C\gets C/\gamma$ if training with Adam.
\end{itemize}
\end{lemma}
Note the conclusion about Adam in \Cref{lem:mup_freedom} also extends to SignGD.

With \Cref{lem:mup_freedom}, we can always set the multiplier of any weight matrix $M$ to be 1, which leave us only the initialization variance $\sigma^2_M$ and learning rate $\eta_M$. Furthermore, in terms of the scale at initialization and the scale of updates, $\mup$ for SGD and SignGD are entirely the same. The only difference would be learning rate. We provide details in \Cref{tab:scale_pre_training} (recall $M_{-1}$ is the weight $M$ at initialization of pre-training, $\Delta M_0=M_0-M_{-1}$ is the overall change of weight in pre-training. We further assume $\chi_t=\Theta(n^{-d})$ for all $t$, thus $\eta_M n^{d}$ is the scale of learning rate for SignGD in pre-training).

\begin{table}[!htbp]
    \centering
    \begin{tabular}{c|c|c|c}
    coordinate-wise scale     & $M=U^i$ & $M=W^j$ & $M=V$ \\\hline
       $M_{-1}$   & $\Theta(1)$ & $\Theta(1/\sqrt n)$ & $\Theta(1/n)$ \\\hline
       $\Delta M_0$ & $\Theta(1)$ & $\Theta(1/ n)$ & $\Theta(1/ n)$ \\\hline
       $\eta_M$ for SGD & $\Theta(n)$ & $\Theta(1)$ & $\Theta(1/n)$ \\\hline
       $\eta_M\cdot n^d$ for SignGD/Adam & $\Theta(1)$ & $\Theta(1/n)$ & $\Theta(1/n)$  \\
    \end{tabular}
    \caption{Scales of initialization, update and learning rate for $\mup$ in pre-training.}
    \label{tab:scale_pre_training}
\end{table}

Since we have different learning rate for different $M$, the kernel that we care is defined as 
\[\gK(\xi, \xi')=\sum_{M\in \gM}\eta_{M}' \left\langle\nabla_W f(\xi), \phi(\nabla_W f(\xi'))\right\rangle,\]
where $\phi$ is identity if the algorithm is SGD, $\phi=\sign$ if the algorithm is SignGD, $\eta'_M=\eta_M$ for SGD, $\eta'_M=\eta_M n^d$ for SignGD. We use $\eta_M'$ to keep $\gK(\xi,\xi')=\Theta(1)$.

And we want to prove the dynamic of the network follows
\[\frac{f_t(\xi) - f_{t-1}(\xi)}{n^{-d}}\to -\tilde\chi_t\gK(\xi,\xi_t) \quad \text{ when }n\to \infty,\]
where $\tilde\chi_t=n^{-d} \chi_t$ for SGD, and $\tilde\chi_t=\sign(\chi_t)$ for SignGD. In any case, $\tilde\chi_t=\Theta(1)$.

\subsection{Prompt-based Fine-Tuning: A Linear Example}
\label{sec:app_linear_intuition}

As an intuitive example, we consider a three-layer linear network
\[f(\xi; U, W, V) = V^\top WU\xi.\]
For simplicity, we train the network with SGD, and freeze $V$ so $\eta_V=0$. Then we have $\nabla_U f=W^\top V \xi^\top$ and $\nabla_W f=V (U\xi)^\top$. We assume $|\langle \xi,\xi'\rangle|>0$ for any $\xi, \xi'$.

In what follows, we will prove that for pre-training $f$ cannot be written as the first-order Taylor expansion (i.e., it exhibits feature learning). Then we will prove that it is the opposite for fine-tuning. In fact, if we only look at one gradient step, the only higher order term equals to $\eta_W\eta_U\chi_t^2\|V\|^2\langle \xi_t, \xi\rangle f_{t-1}(\xi)=\Theta(\chi_t^2f_{t-1}(\xi))$, where $f_{t-1}(\xi)$ is mostly $\Theta(1)$, $\chi_t$ is mostly $\Theta(1)$ in pre-training\footnote{$f_t(\xi)$ is $\Theta(1)$ unless $t=-1$ or there are coincidental cancellations. $\chi_t$ is $\Theta(1)$ in pre-training until $f$ memorizes the whole pre-training dataset when $n\to\infty$.} and $o(1)$ in fine-tuning (by \Cref{def:solvable_task}).

\paragraph{Zero step (Pre-training)} We model the pre-training of $f$ as one step of training with $\chi_0=\Theta(1)$. Then we have 
$\Delta U_0 = -\eta_U \chi_0 W_{-1}^\top V\xi_0^\top$, 
and $\Delta W_0 = -\eta_W \chi_0 V(U_{-1}\xi_0)^\top$.
Since $W_{-1}^\top$ is independent from $V$, we have $W_{-1}^\top V=\Theta(1/n)$, thus $\Delta U_0 =\Theta(1)$ matching \Cref{tab:scale_pre_training}.
On the other hand, it is obvious that $\Delta W_0=\Theta(1/n)$ because $V=\Theta(1/n)$ and $U=\Theta(1)$, also matching \Cref{tab:scale_pre_training}.

Then the function is now 
\begin{align*}
    f_0(\xi)=&~V^\top (W_{-1}+\Delta W_0)(U_{-1} + \Delta U_0)\xi \\
    =&~V^\top (W_{-1}-\eta_W \chi_0 V(U_{-1}\xi_0)^\top)(U_{-1}\xi -\eta_U\chi_0 W^\top_{-1}V\langle\xi_0,\xi\rangle) \\
    =&~V^\top W_{-1}U_{-1}\xi-\eta_U\chi_0 \|W^\top_{-1}V\|_2^2\langle\xi_0,\xi\rangle -\eta_W \chi_0 \|V\|^2\langle U_{-1}\xi_0,U_{-1}\xi\rangle\\
    &~+\eta_W \eta_U\chi_0^2 \|V\|^2\langle\xi_0,\xi\rangle V^\top W_{-1}U_{-1}\xi.
\end{align*}
It is not difficult to see that $\eta_U\chi_0 \|W^\top_{-1}V\|_2^2\langle\xi_0,\xi\rangle$, $\eta_W \chi_0 \|V\|^2\langle U_{-1}\xi_0,U_{-1}\xi\rangle$, and $\eta_W \eta_U\chi_0^2 \|V\|^2\langle\xi_0,\xi\rangle$ are all $\Theta(1)$. Unfortunately, here $V^\top W_{-1}U_{-1}\xi=f_{-1}(\xi)=o(1)$ in the infinite-width limit, but if we train one more step, it is easy to see that all four terms of $f_0$ is $\Theta(1)$. Therefore, pre-training with $\mup$ exhibits feature learning.

\paragraph{First step} At the first step of fine-tuning, we have $\Delta U_1 = -\eta_U \chi_1 W_0^\top V\xi_1^\top$ and $\Delta W_1 = -\eta_W \chi_1 V(U_0\xi_1)^\top$. The function can be written as
\[f_1(\xi)=V^\top (W_0+\Delta W_1)(U_0 + \Delta U_1)\xi,\]
and 
\begin{align}
    f_1(\xi)-f_0(\xi)=&~V^\top \Delta W_1 U_0 \xi+V^\top W_0\Delta U_1\xi+V^\top \Delta W_1 \Delta U_1\xi.\label{eq:expand_prompt_finetuning}
\end{align}

Note that the sum of the first and second terms is exactly $-\chi_1 \gK(\xi, \xi_1)$. 

Plug in $\Delta W_1=-\eta_W \chi_1 V(U_0\xi_1)^\top$ into the first term of \cref{eq:expand_prompt_finetuning}, 
\[V^\top \Delta W_1 U_0 \xi=-\eta_W \chi_1 V^\top  V(U_0\xi_1)^\top U_0 \xi=\Theta(\chi_1),\]
because
\begin{align*}
    (U_0\xi_1)^\top U_0 \xi =&~ (U_{-1}\xi_1+\Delta U_0\xi_1)^\top (U_{-1} \xi+\Delta U_0\xi)\\
    =&~ \langle U_{-1}\xi_1, U_{-1}\xi\rangle -\eta_U\chi_0 \langle \xi_1,\xi_0\rangle f_{-1}(\xi) -\eta_U\chi_0 \langle \xi,\xi_0\rangle f_{-1}(\xi_1) +\|\Delta U_0\|^2\langle\xi_1,\xi\rangle\\
    =&~\Theta(n).
\end{align*}
Plug in $\Delta U_1 = -\eta_U \chi_1 W_0^\top V\xi_1^\top$ into the second term of \cref{eq:expand_prompt_finetuning}, we have
\[V^\top W_0\Delta U_1\xi=-\eta_U \chi_1 V^\top W_0 W_0^\top V\xi_1^\top\xi=\Theta(\chi_1)\]
 because 
 \begin{align*}
 &~V^\top W_0 W_0^\top V=\| (W_{-1}+\Delta W_0)^\top V, (W_{-1}+\Delta W_0)^\top V\|^2_2 \\
 =&~ \|W_{-1}^\top V\|_2^2+\eta_W^2 \chi_0^2\|V\|^4_2\|U_{-1}\xi_0\|^2_2-2\eta_W\chi_0\|V\|^2_2f_{-1}(\xi_0)=\Theta(1/n).
\end{align*}
The third term of \cref{eq:expand_prompt_finetuning} equals
\[\eta_U\eta_W\chi_1^2V^\top V(U_0\xi_1)^\top W_0^\top V\xi_1^\top\xi=\eta_U\eta_W\chi_1^2\|V\|^2\langle \xi_1,\xi\rangle f_0(\xi_1)=\Theta(\chi_1^2),\]
because $f_0(\xi_1)=\Theta(1)$ unlike $f_{-1}(\xi)$ in the ``zero step'' analysis.
Therefore, $\frac{f_1(\xi)-f_0(\xi)}{\chi_1}\to -\gK(\xi,\xi_1)$.

\paragraph{Second step} 
At the second step of fine-tuning, we have $\Delta U_2 = -\eta_U \chi_1 W_1^\top V\xi_2^\top$, and $\Delta W_2 = -\eta_W \chi_1 V(U_1\xi_2)^\top$ and 
\begin{align}
    f_2(\xi)-f_1(\xi)=&~V^\top \Delta W_2 U_1 \xi+V^\top W_1\Delta U_2\xi+V^\top \Delta W_2 \Delta U_2\xi.%
\end{align}
Assuming $\chi_2$ and $\chi_1$ share the same order, then when $n\to\infty$,
\begin{align*}
    \frac{f_2(\xi)-f_1(\xi)}{\chi_2}\to&~V^\top \Delta W_2 U_1 \xi/\chi_2+V^\top W_1\Delta U_2\xi/\chi_2\\
    =&~-\eta_W V^\top  V(U_1\xi_2)^\top U_1\xi-\eta_U V^\top W_1 W_1^\top V\xi_2^\top\xi \\
    \to&~-\eta_W V^\top  V(U_0\xi_2)^\top U_0\xi-\eta_U V^\top W_0 W_0^\top V\xi_2^\top\xi\\
    =&~ -\gK(\xi, \xi_2).
\end{align*}

\paragraph{$t$th step} Same as the second step by noting $\Delta U_t$, $\Delta W_t$ always have smaller order than $\Delta U_0$ and $\Delta W_0$.

\subsection{LoRA FT Exhibits Kernel Behavior}\label{sec:app_lora_kernel_behavior}
Note \Cref{thm:theory_prompt_finetuning} works for any architecture, including LoRA. In order to apply the theorem to LoRA FT, we need to set the initialization and learning rate of the matrices $A$ and $B$ in LoRA correctly so that they satisfy \Cref{assump:tp_nontrivial_stable}.

Here we provide a relatively straightforward way to accomplish this (assuming only intermediate layers use LoRA): 
\begin{itemize}
    \item Let $k = \alpha n$ where $\alpha$ is a small constant irrelevant to $n$.
    \item Let the initialization scale of $A$ be $\Theta(1/\sqrt n)$. 
    \item Let the learning rate of $A$ and $B$ be $\Theta(1)$ for SGD,  $\Theta(n^{-1-d})$ for SignGD / Adam. 
\end{itemize}
In short words, the initialization and learning rate follows $\mup$ as in \Cref{tab:scale_pre_training} by treating $A$ and $B$ as one of $W^{j}$. This setup easily generalizes to the case where $U$ and $V$ also use LoRA.

\section{Subspace-Based Fine-Tuning Methods}\label{sec:app_theory_lora}
Experimental results related to LoRA FT are presented in \Cref{tab:lora}. 
These results show that SGD-FT and SGD-LoRA FT perform similarly in the few-shot setting for many tasks, although the original experiments in~\citet{hu2021lora} focused on Adam.
The closeness of $\ksgd$ and $\ksgdlora$ to their respective fine-tuning methods suggests that FT and LoRA FT can be described by kernel dynamics.
Moreover, we show that $\ksgd$ and $\ksgdlora$ achieve similar performance to each other, providing empirical evidence for the claim in \Cref{thm:lora} that LoRA preserves the kernel.
\begin{table}[!h]
\centering

\resizebox{0.9\textwidth}{!}{
\begin{tabular}{rlcccccc}
\toprule
$k$-shot & Method & \textbf{SST-2} & \textbf{MR} & \textbf{CR} & \textbf{QNLI} & \textbf{RTE} & \textbf{QQP} \\
\midrule
\multirow[t]{4}{*}{16} & SGD-FT & 89.0$_{(1.5)}$ & 83.2$_{(2.4)}$ & 93.3$_{(0.2)}$ & 62.1$_{(3.1)}$ & 60.0$_{(5.5)}$ & 62.1$_{(2.3)}$ \\
 & SGD-LoRA FT & 89.1$_{(0.6)}$ & 82.7$_{(2.0)}$ & 92.6$_{(0.8)}$ & 57.1$_{(3.3)}$ & 58.2$_{(2.9)}$ & 59.8$_{(3.0)}$ \\
 & {\cellcolor{gblue!10}$\ksgd$} & {\cellcolor{gblue!10}88.3$_{(0.3)}$} & {\cellcolor{gblue!10}84.7$_{(1.5)}$} & {\cellcolor{gblue!10}93.2$_{(0.9)}$} & {\cellcolor{gblue!10}60.1$_{(3.3)}$} & {\cellcolor{gblue!10}60.0$_{(4.7)}$} & {\cellcolor{gblue!10}58.2$_{(0.9)}$} \\
 & {\cellcolor{gblue!10}$\ksgdlora$} & {\cellcolor{gblue!10}88.1$_{(0.4)}$} & {\cellcolor{gblue!10}84.9$_{(1.4)}$} & {\cellcolor{gblue!10}93.1$_{(1.0)}$} & {\cellcolor{gblue!10}59.4$_{(3.7)}$} & {\cellcolor{gblue!10}56.2$_{(5.8)}$} & {\cellcolor{gblue!10}58.2$_{(3.2)}$} \\
\multirow[t]{4}{*}{64} & SGD-FT & 89.7$_{(0.4)}$ & 85.6$_{(1.1)}$ & 94.3$_{(0.5)}$ & 72.8$_{(2.2)}$ & 68.9$_{(2.5)}$ & 69.2$_{(1.3)}$ \\
 & SGD-LoRA FT & 90.0$_{(0.2)}$ & 85.7$_{(1.2)}$ & 93.9$_{(0.7)}$ & 73.8$_{(2.7)}$ & 69.1$_{(1.8)}$ & 68.3$_{(2.4)}$ \\
 & {\cellcolor{gblue!10}$\ksgd$} & {\cellcolor{gblue!10}89.2$_{(1.0)}$} & {\cellcolor{gblue!10}86.4$_{(0.6)}$} & {\cellcolor{gblue!10}93.7$_{(0.4)}$} & {\cellcolor{gblue!10}67.3$_{(1.6)}$} & {\cellcolor{gblue!10}66.5$_{(2.5)}$} & {\cellcolor{gblue!10}66.4$_{(1.7)}$} \\
 & {\cellcolor{gblue!10}$\ksgdlora$} & {\cellcolor{gblue!10}89.2$_{(0.7)}$} & {\cellcolor{gblue!10}85.7$_{(1.5)}$} & {\cellcolor{gblue!10}93.6$_{(0.4)}$} & {\cellcolor{gblue!10}66.0$_{(1.6)}$} & {\cellcolor{gblue!10}63.5$_{(3.5)}$} & {\cellcolor{gblue!10}63.9$_{(4.5)}$} \\
\bottomrule
\end{tabular}
}

\caption{Performance of prompt-based SGD FT and prompt-based SGD-LoRA FT, along with their kernel analogs $\ksgd$ and $\ksgdlora$, on a subset of tasks.
SGD FT and SGD-LoRA FT achieve comparable performance, and $\ksgd$ and $\ksgdlora$ also achieve comparable performance to each other.  We report F1 for QQP and accuracy otherwise, and average the metrics over 5 seeds.
These experiments support \Cref{thm:lora}. }

\label{tab:lora}
\end{table}

\subsection{IntrinsicDimension FT}
We discuss IntrinsicDimension FT~\citep{li2018measuring,aghajanyan2021intrinsic} here. 
When analyzed through the kernel, IntrinsicDimension FT and LoRA FT induce similar transformations in the optimization dynamics, but the former was originally proposed as a way to measure the difficulty of downstream tasks, and the latter was proposed as an alternative fine-tuning method.
\begin{definition}[$\gA$-IntrinsicDimension FT \citep{li2018measuring,aghajanyan2021intrinsic}]\label{def:intrinsic_dim}
    Let $\theta\in\sR^M$ be the model parameters and fix a random projection $\Pi\in\sR^{M\times k}$. 
    Set $\theta$ to $\theta + \Pi \hat{\theta}$, where $\hat\theta\in\sR^k$.
    To fine-tune, fix $\theta$ at its pre-trained value and only train $\hat{\theta}$.
\end{definition}
We show a similar result for IntrinsicDimension FT as for LoRA FT: using a sufficiently large $k\geq \Theta(\log N/\epsilon^2)$ ensures that each element of the kernel is relatively unchanged.
\begin{theorem}[IntrinsicDimension FT preserves $\ksgd$]
    Let $\Pi$ be a random matrix with each entry draw i.i.d from $\gN(0, 1/k)$. Let $\ksgd_{\text{ID}}\in \R^{N\times N}$ be the kernel analog to SGD-IntrinsicDimension FT (\Cref{def:intrinsic_dim}) on a downstream task $\Xi$. Additionally, assume $\ksgd(i,j)\leq c$ for any $i, j\in [N]$.
	Then, $$\Pr\left[\exists i, j \in [N], |\ksgd_{\text{ID}}(i,j) - \ksgd(i,j)| \geq c\epsilon\right] \leq 4N^2\exp(-(\epsilon^2 - \epsilon^3) k/4).$$
	\label{thm:intrinsicdimension_ft}
\end{theorem}

\subsection{Proofs}
A key step of the proof is to show that if $\gA$ FT exhibits kernel behavior, then so does $\gA$-LoRA FT. 
We show this step in \Cref{sec:app_lora_kernel_behavior}, since it invokes the Tensor Programs framework again.
Now that we know FT follows kernel dynamics, we can move to showing how LoRA and IntrinsicDimension FT modify the kernel.

We restate the Johnson-Lindenstrauss lemma, which preserves inner products under random projection.
\begin{lemma}[Corollary of Johnson-Lindenstrauss, \cite{johnson1984extensions}]
	\label{lem:inp_preserve}
	Let $u,v\in\sR^d$ such that $\|u\|^2\leq c$ and $\|v\|^2\leq c$. Let $h(x)=\frac{1}{\sqrt{k}} Ax$, where $A\in\sR^{k\times d}$ with each entry sampled i.i.d. from $\gN(0,1)$ or $\gU(-1,1)$. Then,
	\begin{equation*}
		\Pr[|u\cdot v - h(u)\cdot h(v)| \geq c\epsilon] \leq 4\exp(-(\epsilon^2 - \epsilon^3)k/4)
	\end{equation*}
\end{lemma}

\begin{proof}[Proof for \Cref{thm:intrinsicdimension_ft}]
    Note $\nabla_{\hat\theta} f = \Pi^\top \nabla_\theta f$, and \[\ksgd_{\text{ID}}(i,j)- \ksgd(i,j) = \langle\nabla_{\hat\theta} f(\xi_i; \theta), \nabla_{\hat\theta} f(\xi_j; \theta)\rangle - \langle\nabla_\theta f(\xi_i; \theta), \nabla_\theta f(\xi_j; \theta)\rangle.\] The rest follows \Cref{lem:inp_preserve} by setting $u=\nabla_\theta f(\xi_j; \theta)$, $v = \nabla_\theta f(\xi_i; \theta)$, and union bounding all $i, j$ pairs.
\end{proof}

We can now look at LoRA~\citep{hu2021lora} for a simple fully connected layer. 
The construction modifies each layer independently and only acts on fully connected layers, so this is the only part of the kernel that can change when parametrizing updates as in LoRA. For ease of notation, for any parameter or hidden vector $w$, we use $dw$ to denote $\nabla_w f(\xi; \theta)$,  $dw(i)$ to denote $\nabla_w f(\xi_i; \theta)$, and $w_i$ denotes the resulting $w$ when input is $\xi_i$.
\begin{lemma}[LoRA SGD Kernel]
	Let $h=Wx + BAx$ as defined in the paper, where $x\in\sR^n$, $W\in\sR^{m\times n}$, $B\in\sR^{m\times k}$, and $A\in\sR^{k\times n}$ with $k\ll n$. $B$ is initialized to 0 and $A$ is initialized with i.i.d. zero-mean Gaussian samples. SGD Training with LoRA (i.e., fixing $W$ and allowing $A$ and $B$ to be updated) yields the kernel $\ksgdlora$, whereas full FT with SGD yields the kernel $\gK$:
	$$
		\ksgdlora = dHdH^\top \odot (XA^\top AX^\top) \qquad  \ksgd = dH dH^\top \odot (XX^\top)
	$$
	where $dH\in\sR^{N\times m}$ has $dh(i)$ in the $i$th row and $X\in\sR^{N\times d}$ has $x_i$ in the $i$th row.
	\label{lem:sgd_lora_grads}
\end{lemma}
\begin{proof}
We start by noting the well-known fact that $dW = dh\otimes x$, where $dh$ is the gradient to $h$ and $\otimes$ is the cross product. Thus, $K = dHdH^\top \odot (XX^\top)$. In the LoRA setting, $dA = 0$ and $dB = dh\otimes Ax$. Because we are in the kernel setting, $B=0$ and thus, $dA=0$, throughout training. So, 
$$
	\klora(i,j) = \langle dB(i), dB(j)\rangle = \langle dh(i), dh(j)\rangle \langle Ax_i, Ax_j \rangle.$$
 Analogous reasoning yields
$$ \ksgd(i,j) = \langle dh(i), dh(j)\rangle \langle x_i, x_j\rangle.$$
\end{proof}

\begin{theorem}[$\ksgdlora$ is likely not far from $\ksgd$]\label{thm:theory_lora}
	\label{thm:app_lora_kernel}
	Let $\ksgdlora\in\sR^{N\times N}$ and $\ksgd\in\sR^{N\times N}$ be defined as in \Cref{lem:sgd_lora_grads}. 
 Additionally, assume that $\|dh\|^2\leq c, \|x\|^2\leq c$ for any  $\xi$ in the downstream dataset.
	Then,
	$$\Pr\left[\exists i, j \in [N], |\ksgdlora(i,j) - \ksgd(i,j)| \geq c^2\epsilon\right] \leq 4N^2\exp(-(\epsilon^2 - \epsilon^3) k/4).$$
\end{theorem}
\begin{proof}
By \Cref{lem:sgd_lora_grads},  
\begin{align*}
    |\ksgdlora(i,j) - \ksgd(i,j)| &= |\langle dh(i), dh(j)\rangle (\langle Ax_i, Ax_j\rangle - \langle x_i,x_j\rangle)| \\&\leq c|\langle Ax_i, Ax_j\rangle - \langle x_i,x_j\rangle|.
    \end{align*}
    The rest of the proof follows from \Cref{lem:inp_preserve} and union bound.
\end{proof}
\begin{remark}
\Cref{thm:theory_lora} shows when $k\geq 20c^4\log N /\epsilon^2$, with high probability, the difference between the two kernels is smaller than $\epsilon$. 
Although \Cref{thm:theory_lora} focuses on a simple fully connected layer, the conclusion easily extends to the case where LoRA is applied $L$ times in the model because LoRA components are independent of each other:
$$\Pr\left[\exists i, j \in [N], |\ksgdlora(i,j) - \ksgd(i,j)| \geq Lc^2\epsilon\right] \leq 4 N^2\exp(-L(\epsilon^2 - \epsilon^3) k/4).$$
The requirement of $k$ becomes $k\geq \Theta(Lc^4\log N /\epsilon^2)$.
\end{remark}

\end{document}